\RequirePackage{tikz}
\documentclass[sn-mathphys,iicol]{sn-jnl}%

\jyear{2023}

\usepackage{amsmath}
\usepackage{amssymb}
\usepackage{mathtools}
\usepackage{enumitem}
\usepackage{breqn}

\usepackage{caption, subcaption}

\usepackage{accents}
\usepackage{hyperref}

\newtheorem{definition}{Definition}

\newtheorem{theorem}{Theorem}
\newtheorem{lemma}{Lemma}

\newtheorem{property}{Property}
\newtheorem{remark}{Remark}\theoremstyle{remark}

\DeclareMathOperator*{\argmax}{argmax}
\DeclareMathOperator*{\maxargmax}{(\, max \, , \, argmax \,)}
\DeclareMathOperator*{\minargmin}{(\, min \, , \, argmin \,)}

\DeclareRobustCommand{\sbullet}{\accentset{\hbox{\fontfamily{lmr}\fontsize{7}{0}\selectfont\textbullet}}}

\newcounter{hdps}
\newenvironment{empty_env}[1][htb]
{
\refstepcounter{hdps}

}

\DeclarePairedDelimiter\floor{\lfloor}{\rfloor}

\usepackage[algo2e,linesnumbered,ruled,vlined]{algorithm2e}
\SetKwRepeat{Do}{do}{while}

\begin{document}
\title{Assessing hierarchies by their consistent segmentations}

\author[1]{Zeev Gutman}\email{szeg25@gmail.com}

\author[2]{Ritvik Vij}\email{ritvikvi@amazon.com}
\author*[3]{Laurent Najman}\email{laurent.najman@esiee.fr}
\author[4]{Michael Lindenbaum}\email{mic@cs.technion.ac.il}

\affil[1]{\orgname{Rafael}, \orgaddress{\country{Israel}}}
\affil[2]{\orgdiv{Amazon}, \orgname{India}}
\affil[3]{\orgname{Univ Gustave Eiffel, CNRS}, \orgdiv{LIGM}, \orgaddress{\postcode{F-77454} \city{Marne-la-Vall\'ee},  \country{France}}}
\affil[4]{\orgname{CS dept.},\orgdiv{Technion}, \orgaddress{\city{Haifa}, \country{Israel}}}

\abstract{
Current  approaches to generic segmentation start by creating a hierarchy of nested image partitions and then specifying a segmentation from it. 
Our first contribution is to describe several ways, most of them new, for specifying segmentations using the hierarchy elements. 
Then, we consider the best hierarchy-induced segmentation specified by a limited number  of hierarchy elements. 
We focus on a common quality measure for binary segmentations, the Jaccard index (also known as IoU).  Optimizing the Jaccard index is highly non-trivial, and yet we propose an efficient approach for doing exactly that. This way we get algorithm-independent upper bounds on the quality of any segmentation created from the hierarchy. We found that the obtainable segmentation quality varies significantly depending on the way that the segments are specified by the hierarchy elements, and that representing a segmentation with only a few hierarchy elements is often possible. (Code  is available). 
}

\keywords{Hierarchical segmentation -- Image segmentation -- Evaluation -- Jaccard index}
\maketitle

\section{Introduction}\label{sec:introduction}

Generic ({\em i.e.}, non-semantic) image segmentation is widely used in various tasks of image analysis and computer vision. 
A variety of image segmentation methods are proposed in the literature, including the watershed method \cite{beucher1979international}, level-set method \cite{osher1988fronts}, normalized cuts \cite{shi2000normalized}, and many others. 
Modern generic segmentation algorithms use (deep) edge detectors and watershed-like merging \cite{maninis2017convolutional}. Augmenting the detected edges with region descriptors improves segmentations  \cite{isaacs2020enhancing}. Note that generic image segmentation, the topic we are focusing on in this paper, is different from 
semantic image segmentation, which provides segmentation of objects from specific classes with the help of (deep) image classifiers \cite{lateef2019survey,minaee2021image}. 

Segmentation (generic or semantic) is useful for numerous applications, such as image enhancement~\cite{jam2020comprehensive}, image analysis \cite{redmon2016you}, and medical image analysis~\cite{chen2018survey}.

The dominant generic segmentation algorithms ({\em e.g.}, \cite{maninis2017convolutional}) are hierarchical and built as follows: first, an oversegmentation is carried out, specifying superpixels as the elements to be grouped. Then a hierarchical structure (usually represented by a tree) is constructed with the superpixels as its smallest elements ({\em i.e.}, leaves). The regions specified by the hierarchy are the building blocks from which the final segmentation is decided.  Restricting the building blocks to the elements of the hierarchy yields simple, effective algorithms at a low computational cost. Most segmentation methods build the segmentation from the hierarchy by choosing a cut from a limited cut set. Our first contribution is to generalize this choice. We systematically consider all possible ways for specifying a segmentation, using set operations on elements of the hierarchy. Most of these methods are new.

We are also interested in the limitations imposed on the segmentation quality by using the hierarchy-based approach. These limitations depend on (1) the quality of the hierarchy, (2) the number of hierarchy elements (nodes) that may be used, and (3) the way that these elements are combined. We investigate all these causes in this paper. The quality is also influenced by the oversegmentation quality, which was studied elsewhere~\cite{achanta2012slic}.

The number of hierarchy elements determines the complexity of specifying a segmentation. 
Lower complexity is advantageous by the minimum description length (MDL) principle, which  minimizes a cost composed of the description cost and the approximation cost, and relies on statistical justifications 
\cite{cook1993substructure,grunwald2007minimum,rissanen1978modeling,quinlan1989inferring,veras2016optimizing}.
Moreover, representation by a small number of elements opens possibilities for a new type of segmentation algorithms that are based on search, for example, in contrast to the greedy current algorithms.
The number of elements needed also indicates, in a sense, how much information about the segmentation is included in the hierarchy, and thus, it provides a measure of quality for the hierarchy as an image descriptor, as well as a global measure of the associated boundary operator.

To investigate the {\em hierarchy-induced  limitations}, we optimize the segmentation from elements of a given hierarchy. We consider binary segmentation, and use the Jaccard index (IoU) measure of quality \cite{jaccard1901etude}. More precisely, we use image-dependent  oversegmentation and hierarchies produced by algorithms that have access only to the image. However, we allow the final stage, which constructs the segmentation from the hierarchy elements, to have access to the ground-truth segmentation. As a result, the imperfections of the optimized segmentation correspond only to its input, {\em i.e.}, to the hierarchy. Thus, the results we obtain are upper bounds on the quality that may be achieved by any realistic algorithm, that does not have access to the ground truth, but relies on the same hierarchy. 

Optimizing the Jaccard index is highly nontrivial, but we provide a framework that optimizes it exactly and effectively.  
Earlier studies either use simplistic quality measures or rely on facilitating constraints~\cite{pont2012upper, pont2012supervised}. 

The contributions of this work are:
\begin{enumerate}[noitemsep]
\item Four different methods
for specifying a hierarchy-induced segmentation. These methods are denoted 
(segmentation to hierarchy)  {\em consistencies}.

\item Efficient and exact algorithms for finding the best segmentation (in the sense of maximizing the Jaccard index) that is consistent with a given hierarchy. We provide four algorithms\footnote{Code is available at \href{https://github.com/ritvik06/Hierarchy-Based-Segmentation}{https://github.com/ritvik06/Hierarchy-Based-Segmentation}}, one for each consistency. 
The algorithms are fast, even for large hierarchies. 

\item A characterization of the limits of hierarchy-induced segmentation. Notably, this characterization is also a measure of the hierarchy quality.
\end{enumerate}

This paper considers segmentation of images, but all the results apply as well to the partition of general data sets. The paper continues as follows. First, we describe terms and notations required for specifying the task (Section \ref{sec:preliminaries}). In Section \ref{sec:problemformulation}, we present our goal and discuss the notion of consistencies, which is central to this paper. In Section \ref{sec:previouswork}, we review several related works. In Section \ref{sec:ourapproach}, we develop an indirect   optimization approach that relies on the notion of co-optimality and enables us to optimize certain quality measures.  Section 
\ref{sec:Jaccard} 
provides particular optimization algorithms and the corresponding upper bounds for the Jaccard index and the different consistencies. The bounds are evaluated empirically in Section \ref{sec:experiments}, which also provides some typical hierarchy-based segmentations.
Finally, we conclude and suggest some extensions in Section \ref{sec:conclusions}.

\begin{figure*}[t]
	\def\svgwidth{1\textwidth}
    \centering{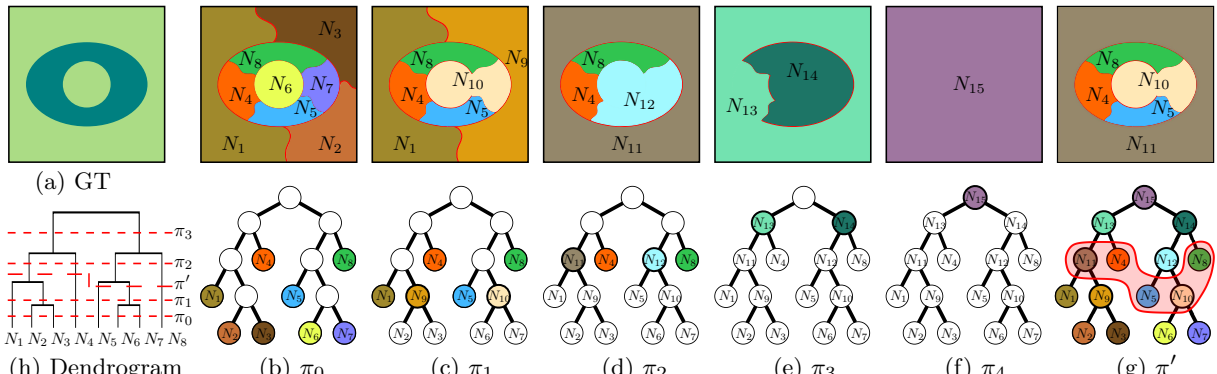}
	\caption{{\bf (a)} The true segmentation (GT). {\bf (b-f)} A chain of the image partitions: $\Pi = \{\pi_0, \, \pi_1, \, \pi_2, \, \pi_3, \, \pi_4\},$ which yields a hierarchy ${\cal T} = \{N_1, \dots , N_{15}\}$. Each $\,\pi_i\,$ is represented in the Binary Partition Tree (representing $\,\cal T$) by a set of colored nodes. {\bf (g)} Another partition of the image, denoted $\,\pi'$. The nodes representing $\,\pi'\,$ (shaded red) are a cut of the hierarchy $\,\cal T,$ and are the leaves of a tree $\,\cal T'$ obtained by pruning $\,\cal T.$ {\bf (h)} The dendrogram representing $\,\cal T.$ Each partition of the above is represented by a cut of the dendrogram (red dashed lines).}
	\label{fig:hierarchy-examp}
\end{figure*}

\section{Preliminaries}
\label{sec:preliminaries}
\subsection{Hierarchies}\label{sec:introduction:hierarchies}

The following definitions and notations are standard, but are presented here for the sake of completeness. Recall that a partition of a set I is a set of non-empty subsets of I, such that every element in I is in exactly one of these subsets ({\em i.e.}, I is a disjoint union of the subsets). In this paper, these subsets are referred to as regions. Moreover,  all examples are done with connected regions, but the connectivity constraint is not needed for the theory and algorithms. 

Let $\pi_1$ and $\pi_2$ be two partitions of a pixel set $I.$ Partition $\pi_1$ is {\em finer} than partition $\pi_2,$ denoted $\pi_1\leq \pi_2,$ if each region of $\pi_1$ is included in a region of $\pi_2.$ In this case, we also say that $\pi_2$ is {\em coarser} than $\pi_1.$ Let $\Pi$ be a finite chain of partitions $\Pi\!=\!\{ \pi_i \,\, | \,\, 0 \leq i \leq j \leq n \implies \pi_i \leq \pi_j \}$ where $\pi_0$ is the finest partition and $\pi_n$ is the trivial partition of $\,I\,$ into a single region: $\pi_n = \{I\}$. A hierarchy $\,\cal{T}$ is a pool of regions of $\,I,$ called {\em nodes}, that are provided by elements of $\Pi\!:\,\,{\cal T} \!=\! \{\, N \! \subset I \,\, | \,\, \exists \, \pi_i \in \Pi \, : \, N \! \in \! \pi_i \,\}.$ For any two partitions from $\Pi,$ one is finer than the other, hence, any two nodes $N_1, N_2 \!\in\! \cal T$ are either nested $(N_1 \!\subset\! N_2$ or $N_2 \!\subset\! N_1),$ or disjoint $(N_1 \!\cap\! N_2 \!= \emptyset);$ see Figure \ref{fig:hierarchy-examp}.

Let $N_1$ and $N_2$ be two different nodes of $\,\cal T$. We say that $N_1$ is the {\em parent} of $N_2\,$ if $\,N_2 \subset N_1$ and there is no other node $N_3 \! \in \! {\cal T}$ such that $N_2 \! \subset \! N_3 \! \subset \! N_1$. In this case, we also say that $N_2$ is a {\em child} of $N_1$. Note that every node has exactly one parent, except $I \in \pi_n\,,$ which has no parent. Hence, for every node $N \! \in \! {\cal T}$, there is a unique chain: $N \! = \! N_1 \! \subset \! \dots \! \subset \! N_k \! = \! I\,,$ where $\,N_i\,$ is the parent of $\,N_{i-1}$. Thus, the parenthood relation induces a representation of $\,\cal T$ by a tree, 
in which the nodes of $\,\pi_0\,$ are the leaves, and the single node of $\,\pi_n\,$ is the root; see Figure \ref{fig:hierarchy-examp}. Hence, we also refer to $\,\cal T$  as a tree. %
When each non-leaf node\footnote{Note that the term node may refer to the node in the tree but also to the corresponding image region,  when the context is clear.} of $\,\cal T$ has exactly two children, $\cal T$ is a {\em binary partition tree} (BPT) \cite{salembier2000binary, lu2007binary, pont2012upper, pont2012supervised}. In this paper, we focus on BPTs, but our results hold for non-binary trees as well.

A hierarchy $\cal T$ can be represented by a dendrogram, and every possible partition of $I$ corresponds to a set of $\cal T$'s nodes and may be obtained by “cutting” the dendrogram; see Figure \ref{fig:hierarchy-examp}. In the literature, any partition of $I$ into nodes of $\cal T$ is called a {\em cut of the hierarchy} \cite{guigues2006scale, xu2016hierarchical}. Every $\pi_i \in \Pi$ is  a {\em horizontal cut} of the hierarchy, but there are many other ways to cut the hierarchy, and each cut specifies a partition of $\,I.$ 
As we shall see later, a hierarchy may induce other partitions of $I$. %

Pruning of a tree in some node $N$ is a removal from the tree of the entire subtree rooted in $N,$ except $N$ itself, which becomes a leaf. Each cut of a hierarchy represents a tree $\,\cal T'$ obtained by pruning $\,\cal T,$ by specifying the leaves of $\,\cal T'$; see Figure \ref{fig:hierarchy-examp}. The converse is also true: the leaves of a tree obtained by pruning $\,\cal T$ are a cut of the hierarchy. That is, a subset of nodes $\,\cal N \! \subset \! T$ is a cut of the hierarchy, if and only if $\,\cal N$ is the set of leaves of a tree obtained by pruning $\,\cal T$. More precisely, $\,\cal N \! \subset \! T$ is a cut of the hierarchy, if and only if for every leaf in $\,\cal T$, the only path between it and the root contains exactly one node from $\,\cal N$. Often, a segmentation is obtained by searching for the best pruning of $\,\cal T$. However, the cardinality of the set of all prunings of $\,\cal T$ grows exponentially with the number of leaves in $\,\cal T$ \cite{pont2012upper}. Thus, it is unfeasible to scan this set exhaustively by brute force.

\subsection{Coarsest partitions}\label{sec:introduction:coarsest}

We use the following notations. The cardinality of a set is denoted by $|\, \cdot \,|$. The initial partition $\pi_0$ of $\,I,$ which is the set of leaves of the tree $\,\cal T,$ is  denoted by $\,\cal L.$ Let $N \! \in \! \cal T ; \,$ we denote by ${\cal T}^N \!\! \subset \! \cal T$ (resp. ${\cal L}^N \!\! \subset \! \cal L\,$) the subset of nodes of $\,\cal T$ (resp. $\cal L$) included in $N$. Note that $\,{\cal T}^N$ is represented by the subtree of $\,\cal T$ rooted in $N$; hence, we refer to $\,{\cal T}^N$ also as a subtree, and to $\,{\cal L}^N$ as the leaves of this subtree.

Let $\,Y \! \subset I$ be a pixel subset. We refer to any partition of $Y$ into nodes of $\cal T$ (namely, a subset of disjoint nodes of $\,\cal T$ whose union is $\,Y$) as a {\em$\,\cal T\!$-partition of $\,Y$}. Note that a $\cal T$-partition of $\,Y$ does not necessarily exist. We refer to the smallest subset of disjoint nodes of $\,\cal T$ whose union is $\,Y\,$ as the {\em coarsest $\,\cal T\!$-partition of $\,Y$}. Obviously, $\,\cal N \! \subset \! T\,$ is a cut of the hierarchy if and only if $\,\cal N\,$ is a $\,\cal T$-partition of $\,I$. 
$\,\cal N \! \subset \! T\,$ is a $\,\cal T$-partition of a node $N \! \in \! \cal T,\,$ if and only if $\,\cal N$ is the set of leaves of a tree obtained by pruning $\,{\cal T}^N\!$. Obviously, the coarsest $\,\cal T$-partition of a node $N \! \in \! \cal T$ is $\{N\}.$

Figure \ref{fig:hierarchy-disj-compl} illustrates several ways of representing a region using a hierarchy and the corresponding coarsest partition. 

\begin{property}\label{prop:T-partition}
	 A non-coarsest $\,\cal T$-partition of a node $N \! \in \! \cal T$ is a union of $\,\cal T$-partitions of its children.
\end{property}
\noindent
In Figure \ref{fig:hierarchy-examp}(g), for example, the subset $\{N_4, N_5, N_8, N_{10}, N_{11}\}$ is a non-coarsest $\cal T$-partition of $N_{15}\,$ whereas, $\{N_4, N_{11}\}$ and $\{N_5, N_8, N_{10}\}$ are $\cal T$-partitions of the children of $N_{15}: \, N_{13}$ and $N_{14},$ respectively.

\begin{lemma}
\label{lemma:coarsest-T-partition}
(See Appendix \ref{lemma:coarsest-T-partition:Proof} for the proof.)
\begin{enumerate}
\item[\it i.] A $\,\cal T$-partition of a pixel subset $\,Y \! \subset I\,$ is non-coarsest, if and only if it contains a non-coarsest $\,\cal T$-partition of some node $\,N \! \in \! \cal T$ that is included in $\,Y$ ($\,N \! \subset Y$).
			
\item[\it ii.] When the coarsest $\,\cal T$-partition of a pixel subset $\,Y \! \subset I\,$ exists, it is unique.
\end{enumerate}
\end{lemma}
  
\section{Problem formulation}\label{sec:problemformulation}

\subsection{The general task }\label{sec:task}

As discussed in the introduction, we consider only segmentations that are consistent in some way with a given hierarchy, and aim to find  limitations on their quality. 

Obviously, the quality improves with the number of regions. 
More precisely, 

\noindent{\normalsize\bf General task:} {\em Given a hierarchy and a measure for estimating segmentation quality, we want to find a segmentation that has the best quality, is consistent with the hierarchy, and uses no more than a given number of regions from it.}

To make this task well-defined, we now specify and formalize the notion of segmentation consistency with a hierarchy.

\subsection{Various types of consistency between a segmentation and a hierarchy}\label{sec:Consistensies}

We consider segmentations whose (not necessarily connected) segments are specified by set operations on the nodes of $\,\cal T$. The intersection between two nodes in a hierarchy is either empty or one of the nodes. Therefore, we are left with union and set-difference. Complementing (with respect to $I$) is  allowed as well. By further restricting the 
particular operations and the particular node subsets on which they act, we get different, non-traditional, ways for specifying segmentation from a hierarchy. We denote these different ways as (hierarchy to segmentation) consistencies.  

\begin{definition}{\bf  Consistency.}
Let $\,\cal Y\,$ be a set of pixel subsets; we denote by $\,\sbullet{\cal Y} \subset I$ the union of all elements of $\,\cal Y.$ 
We say that: 
\begin{enumerate}
	\itemsep2pt%
	
	\item[(a)]   A segmentation $s$ is {\em  a-consistent} with $\cal T$ if there is a subset $\,{\cal N}_s \! \subset \! \cal T\,$ such that each segment in $s$ is a single node of $\,{\cal N}_s.$
	
	\item[(b)] A segmentation $s$ is {\em  b-consistent} with $\cal T\,$ if there is a subset $\,{\cal N}_s \! \subset \! \cal T\,$ such that each segment in $s$ is a union of some nodes of $\,{\cal N}_s.$
	
	\item[(c)]  A segmentation $s$ is {\em c-consistent} with $\cal T\,$ if there is a subset $\,{\cal N}_s \! \subset \! \cal T\,$ such that each segment in $s,$ except at most one, is a union of some nodes of $\,{\cal N}_s.$ One complement segment, if it exists, is $I \backslash \sbullet{\cal N}_s$.
	
	\item[(d)]  A segmentation $s$ is {\em d-consistent} with $\cal T\,$ if there is a subset $\,{\cal N}_s \! \subset \! \cal T\,$ such that each segment, except at most one, is obtained by unions and/or differences of nodes of $\,{\cal N}_s.$ One complement segment, if it exists, is $I \backslash \sbullet{\cal N}_s.$
	
\end{enumerate}
\end{definition}
\begin{remark}\label{remark:consistencies}
		Consistency of some type,  with the subset ${\cal N}_s$, implies consistency of a later, more general type, with the same subset ${\cal N}_s$. 
\end{remark}

\begin{remark}\label{remark:generality}
We argue that these four consistency types systematically cover all possibilities. The first choice is whether the nodes subset ${\cal N}_s$ should be limited to a hierarchy cut or not. For the cut case (which is the popular choice in the literature), union is the only set operation that makes sense, because set difference between disjoint nodes is empty and the cut covers the full image, making the complement empty as well. For the more general case, where the subset ${\cal N}_s$ is not necessarily a cut,  both the union and the set difference are relevant. Unions without set difference is an important special case that is simpler both conceptually and computationally. Set difference between two nodes without additional unions does not seem to justify another consistency type (and is included, of course, in d-consistency). 
\end{remark}

Figure \ref{fig:hierarchy-segm-examp} illustrates the different consistencies.
The a-consistency is used in most hierarchy-based segmentation algorithms, where some cut is chosen and all its leafs are specified as segments; see    \cite{pont2012supervised, pont2016supervised}.
To the best of our knowledge, the b-, c-, and d- consistencies were not used in the context of hierarchical segmentation; see however \cite{passat2011selection} for (c-consistency-like) node selection in a  hierarchy of components. 

As specified above, the subset $\,{\cal N}_s$, specified for a segmentation $s$, is not necessarily unique; see Figures \ref{fig:hierarchy-disj-compl} and \ref{fig:hierarchy-segm-examp}(c,d). From  this point forward in this paper, $\,{\cal N}_s$ is considered as the minimal set so that all nodes in it are required to specify $s$. %
As a result, 
\begin{property}\label{prop:consistencies}
	
	If ${\cal N}_s \!\subset\! \cal T$ is a subset associated with some consistency type a/b/c/d of a segmentation $s,$ then
	\begin{enumerate}
		\itemsep2pt%
		
		\item[(a)] $s\,$ is a-consistent with $\,\cal T,\,$ if and only if $\,{\cal N}_s\,$ is a cut of $\,\cal T$ such that each segment of $s$ is a single node of $\,{\cal N}_s.$ The subset $\,{\cal N}_s\,$ associated with a-consistency of $s$ is unique.
		
		\item[(b)] $s\,$ is b-consistent with $\,\cal T,\,$ if and only if $\,{\cal N}_s\,$ is a cut of $\,\cal T.$
		
		\item[(c)] $s\,$ is c-consistent with $\,\cal T,\,$ if and only if $\,{\cal N}_s\,$ consists of disjoint nodes of $\,\cal T.$
		
		\item[(d)] $s\,$ is d-consistent with $\,\cal T,\,$ if and only if $\,{\cal N}_s\,$ consists of (possibly overlapping) nodes of $\,\cal T.$ 
\end{enumerate}
\end{property}

\begin{lemma}\label{lemma:consist:equiv}
	 Every segmentation that is consistent with a hierarchy in one of the types b/c/d is also consistent with the hierarchy in the other two types.
\end{lemma}
\begin{proof}[\bf Proof sketch:]
	Following Remark \ref{remark:consistencies}, consistency of a segmentation according to one type implies its consistency according to the more general types. The converse is also true.
	
	Consider a d-consistent segmentation $\,s.$ Recall that every node in $\cal T$ is a union of disjoint nodes of the initial partition $\cal L.$ A set-difference of nested nodes is still a union of nodes of $\,\cal L$ (which are a $\cal T$-partition of this set-difference); see Figure \ref{fig:hierarchy-disj-compl}(a-c). Hence, there is a subset ${\cal N}_s$ consisting of disjoint nodes. By Property \ref{prop:consistencies}, $s\,$ is also c-consistent.
	
	A subset ${\cal N}_s$ consisting of disjoint nodes can be completed to a partition of $\,I$ by adding some $\cal T$-partition of $\,I \backslash \sbullet{\cal N}_s\,;$ see Figure \ref{fig:hierarchy-disj-compl}(d-e). Hence, there is another subset ${\cal N}_s$ that is a cut of the hierarchy. By Property \ref{prop:consistencies}, $s$ is also b-consistent.
\end{proof}

Lemma \ref{lemma:consist:equiv} states the somewhat surprising result that  b/c/d consistencies are equivalent. 
Thus, the set of segmentations consistent with $\,\cal T$, using either b-, c-, or d-  consistencies, is common. Denote this set by $\cal S$. Note that the set of a-consistent segmentations, ${\cal S}_1 \!\subset\! \cal S$, is smaller. 
The consistencies may differ significantly, however, in the ${\cal N}_s\,$ subsets. Let  $\,{\cal N}_s^a\,$ (resp. ${\cal N}_s^b\, , \,{\cal N}_s^c\, , \,{\cal N}_s^d\,$) be the smallest subset such that $\,s \!\in\! \cal S$ is a- (resp. b-, c-, d-) consistent with this subset

\begin{figure*}[htb]
	\def\svgwidth{1\textwidth}
	\centering{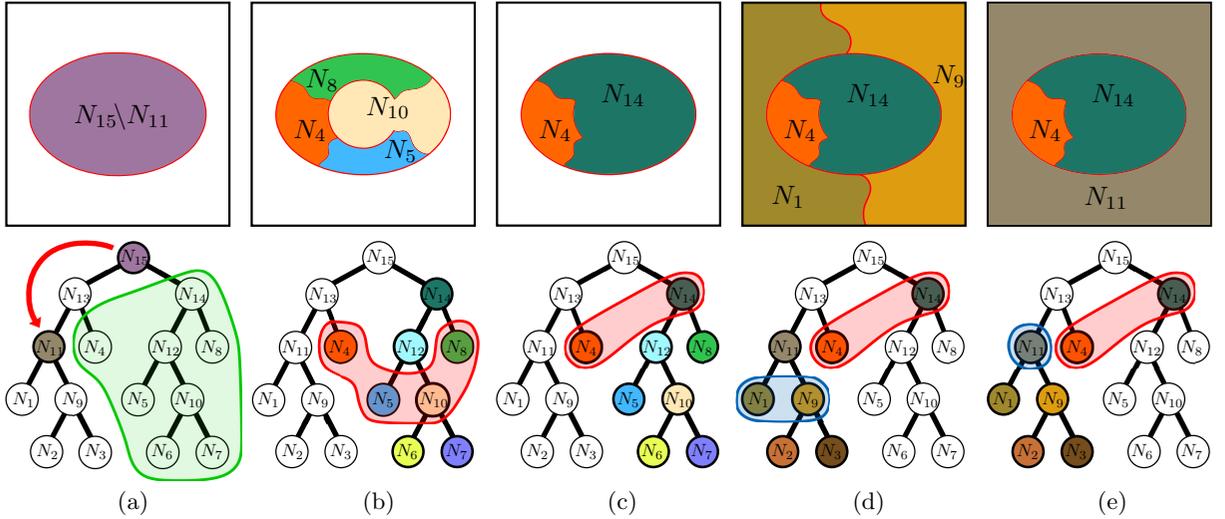}
	\caption{A few examples illustrating how regions may be represented by a hierarchy. We use the hierarchy described in Figure \ref{fig:hierarchy-examp}. {\bf (a)} Set-difference of nodes: $N_{15} \!\backslash N_{11}$. The nodes covering the part of $\,I$ that are included in this set-difference are shaded green. {\bf (b)} A possible $\,\cal T$-partition of $\,N_{15} \!\backslash N_{11}\,$ is shaded red. {\bf (c)} The unique coarsest $\,\cal T$-partition of $\,N_{15} \!\backslash N_{11}\,$, which is $\{N_4, N_{14}\}$, is shaded red. {\bf (d)} A possible $\cal T$-partition of the complement $I \backslash (N_4 \! \cup \! N_{14})$ is shaded blue. {\bf (e)} The unique coarsest $\cal T$-partition of the complement $I \backslash (N_4 \!\cup\! N_{14}),$ which is $\{N_{11}\},$ is shaded blue.%
	}\label{fig:hierarchy-disj-compl}
\end{figure*}

The proof of the following lemma is straightforward. 
\begin{lemma}\label{lemma:consist:inequal}
	 Let $\,s \! \in \! \cal S\,$, then  $|{\cal N}_s^b| \, \geqslant |{\cal N}_s^c| \, \geqslant |{\cal N}_s^d|$. Furthermore, $\,{\cal N}_s^b\,$ is unique, but not necessarily $\,{\cal N}_s^c , \,{\cal N}_s^d.$
	Moreover, if $s$ is a-consistent, then 
	${\cal N}_s^a= {\cal N}_s^b$.
\end{lemma}

Note that a segmentation is consistent with $\,\cal T$ (in some consistency a/b/c/d), if and only if each of its segments is a union of nodes from $\,\cal T$; that is, there is a $\,\cal T$-partition for each segment.
For a segmentation $s \! \in \! \cal S,$ 
we refer to the union of the coarsest $\,\cal T$-partitions of the segments of $\,s$ ({\em i.e.}, ${\cal N}_s^b$) as {\em the coarsest cut of the hierarchy for} $\,s.$ 
Lemma \ref{lemma:consist:inequal} implies that for every $s \! \in \! \cal S$ there is a unique coarsest cut of the hierarchy. The converse is not true. A cut of a hierarchy can be the coarsest for several segmentations. For example, the same cut is the coarsest for different segmentations (a) and (b)  in Figure~\ref{fig:hierarchy-segm-examp}.


\begin{figure*}[htb]%
	\def\svgwidth{1\textwidth}
  \centering{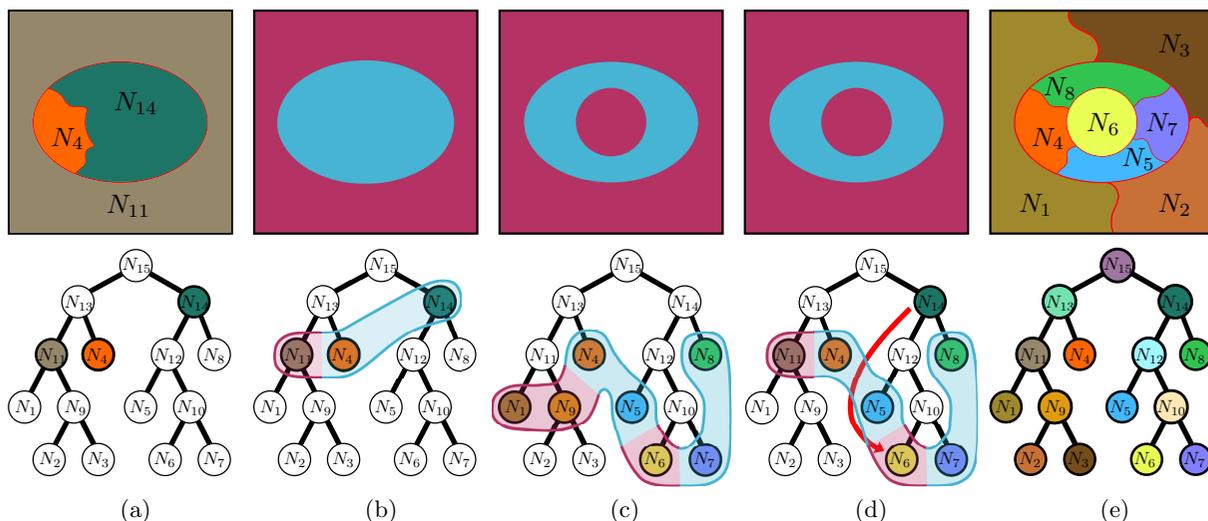}
	\caption{Examples of segmentations of various consistency types, all consistent with the hierarchy described in Figure \ref{fig:hierarchy-examp},  shown also in {\bf (e)}. All segmentations are specified by three nodes (although sometimes fewer nodes suffice). Note that a segment is not necessarily connected. Except for the a-consistency, the nodes in a cut of the hierarchy specifying each segmentation are shaded with the colors of the segments in which they are included. {\bf (a)} An a-consistent segmentation, into three segments, specified by a cut of the hierarchy: $\{N_4, N_{11}, N_{14}\}$. {\bf (b)} A segmentation that is b-consistent with the same cut of the hierarchy as in (a). The nodes $N_4$ and $N_{14}$ are merged into one segment. {\bf (c)} A segmentation, denoted $s$, that is c-consistent with the subset ${\cal N}_s = \{N_1, N_6, N_9\}.$ The burgundy segment is $\,\sbullet{\cal N}_s\,$ and the turquoise segment is the complement $\,I \backslash \sbullet{\cal N}_s.$ Note that the $\,\cal T$-partition of the burgundy segment is non-coarsest. Hence, $\,{\cal N}_s \!\ne\! {\cal N}_s^c\,$ and the specified cut of the hierarchy is non-coarsest for $s.$ The minimal number of disjoint nodes required to cover the turquoise segment is four, while it is only two for the burgundy segment; hence, ${\cal N}_s^c \!=\! \{N_6, N_{11}\}.$ {\bf (d)} The same segmentation $s\,$ is d-consistent with the subset $\,{\cal N}_s = \{N_4, N_6, N_{14}\}.$ The turquoise segment is specified by $N_4 \cup \{N_{14} \!\backslash N_6\},$ the burgundy segment is the rest of the image: $I \backslash \sbullet{\cal N}_s.$ The specified cut of the hierarchy is the coarsest for $s.$ Note that representing this segment with ${\cal N}_s^c \!=\! \{N_6, N_{11}\},$ as specified above, is valid and more node-economical.
	}\label{fig:hierarchy-segm-examp}
\end{figure*}


\section{Previous work}\label{sec:previouswork}

The task considered here (and in  \cite{ge2006image, pont2012upper}) is to estimate the limitation associated with hierarchy-based segmentation. 
That is, to find the best $s \!\in\! \cal S,$ maximizing the quality ${\cal M}(s)$, that is consistent with the hierarchy for a limited size of $\,{\cal N}_s,$ $\,|{\cal N}_s| \,\leqslant k.$ 
This upper-bound of the segmentation quality is  a function of the consistency type and $k$. We refer to the segmentation maximizing the quality as a/b/c/d-optimal.

First, we  emphasize again that this task is different from the common evaluation of hierarchy-dependent segmentations, which provides precision recall curves and chooses the best segmentation from them; see, {\em e.g.},  \cite{arbelaez2011contour, pont2016supervised, perret2017evaluation}.  This approach considers only the easily enumerable set of segmentations associated with horizontal cuts, which are parameterized by a single  scalar parameter. Here, on the other hand, we find the best possible segmentation from much more general segmentation sets, and provide an upper bound on its quality measure.
The best segmentations from these larger sets have often significantly better quality; see \cite{perret2017evaluation}. 

Only a few papers address such upper bounds. 
Most of the upper bounds were derived for local measures. A local 
measure ${\cal M}(s)$ of a segmentation $s$ may be written as a sum of functions defined over the components of the cut defining $s$. 

Local measures are considered in \cite{pont2012upper}. An elegant dynamic programming algorithm provides upper bounds on these measures for segmentations that are  a-consistent  with a given  BPT hierarchy.
Unlike that work, we consider binary segmentation, for which the a-consistent segmentation is trivial. We extend this work by working with b,c, and d-consistent segmentation and by optimizing each one for a non-local measure: the Jaccard index.

The boundary-based $F_b$ measure \cite{martin2003empirical} was considered in \cite{pont2012supervised}. A method to evaluate the a-consistency performance of a BPT hierarchy is proposed.
The optimization was modeled as a Linear Fractional Combinatorial Optimization problem \cite{radzik1992newton} and was solved 
for every possible size of a cut of a hierarchy (from $1$ till $|\cal L|$). This process is computationally expensive, and therefore is limited to moderate size hierarchies.

Extending those previous works, a hierarchy evaluation framework was proposed \cite{perret2017evaluation}. It includes  various types  of upper bounds corresponding to boundaries and regions, and further extends the analysis to  supervised, markers based, segmentation.
More recently, the study described in \cite{randrianasoa2021supervised} introduced some new measures that quantify the match between hierarchy and ground truth. Both papers \cite{perret2017evaluation,randrianasoa2021supervised}  address neither the exact optimization of the Jaccard index nor the advanced (b, c, d) consistencies.  

\section{A co-optimality tool for optimization}\label{sec:ourapproach}

Given a quality measure $\cal M$ over a set $\cal S,$ we want to find $s \! \in \! \cal S$ with the best score, ${\cal M}(s).$ Optimizing the quality measures over all possible node subsets may be computationally hard. One approach could be to optimize an {\em equivalent measure} ${\cal Q}(s)$ instead. Measures are equivalent if they rank objects identically. For example, the Jaccard index and the object-based $F$-measure are equivalent  \cite{pont2016supervised} because they are functionally related by a monotonically increasing function. 

An equivalent measure $\cal Q$ may, however, be as difficult to optimize. Recalling that we are interested only in the maximum of $\cal M$ and not in the ranking of all subsets, we may turn to a weaker, easier-to-optimize form of equivalence. 

\begin{definition}
Let $\,{\cal S}_{_{\!\!{\cal M}}}\! \subset \cal S$ be the subset of the elements optimizing $\cal M.$ We refer to measures $\cal M$ and $\cal Q\,$ as {\em co-optimal} over $\cal S,$ if $\,{\cal S}_{_{\!\!{\cal M}}}\! = {\cal S}_{_{\!\!{\cal Q}}}.$ 
\end{definition}

We now propose an optimization approach that is valid for general finite sets $\cal S$, including but not limited to hierarchical segmentations. 
Algorithm \ref{sch:main} uses a family of measures  $\{{\cal Q}_{_{\omega}}\},\omega \!\in\! [0,1]$ over $\,\cal S$. It works by iteratively alternating between assigning values to $\omega$ and optimizing ${\cal Q}_{_{\omega}}(s)$. As Theorem 
\ref{sch:main:theorem} below shows, under some conditions on the family $\{{\cal Q}_{_{\omega}}\}$, the algorithm returns the segmentation that maximizes the quality measure $\cal M$, and the corresponding maximal value $\widehat{\cal M}$.

	\begin{algorithm2e}[tb]
		\DontPrintSemicolon
		\KwData{%
		A quality measure $\cal M$,  a set $\cal S$, and a family of measures $\{{\cal Q}_{_{\omega}}\}$}
		\KwData{An initial  
		$\,\omega_{_0} \! \in [0,1] %
		\quad$
		}
		\KwResult{An element $\,{s}_{_{\!\omega}}$ }
		$\omega \,=\, {\cal M}(\,\argmax \limits_{\,\,s \, \in \, \cal S} \, {\cal Q}_{_{\omega_{_{_0}}}}\!(s)\,)$\;
		\Do{$\,\omega > \omega_{_0}$}{
			${s}_{_{\!\omega}} = \,\argmax \limits_{\,\,s \, \in \, \cal S} \, {\cal Q}_{_{\omega}}(s)$\;
			$\omega_{_0} = \,\omega$\;
			$\omega \,\, = \,{\cal M}({s}_{_{\!\omega}})$
		}
		\Return $\,{s}_{_{\!\omega}}$
	\caption{\label{sch:main} Generic optimization scheme}

	\end{algorithm2e}

\begin{theorem}\label{sch:main:theorem}

	Let $\cal M$ be a quality measure over a finite set $\,\cal S,$ receiving its values in $[0,1]\,.$
	Let $\widehat{\cal M}$ be the (unknown) maximal value of $\cal M$ over $\cal S$.
   Let $\{{\cal Q}_{_{\omega}}\},\omega \!\in\! [0,1]$ be a family of measures over $\,\cal S,$ satisfying the following conditions:
	\begin{enumerate}
	    \item ${\cal Q}_{\omega = \widehat{\cal M}}\,\,\text{ and }\,\cal M\,$ are co-optimal measures over $\cal S$. \label{sch:main:assump_a}
	    \item
	    For $0 \leqslant \omega < \widehat{\cal M}$  and $ s' \!\in {\cal S},$
	    if there is $\,s \in {\cal S}_{_{\!\!{\cal M}}}\, \text{s.t.}\,{\cal Q}_{_{\omega}}(s) \leqslant {\cal Q}_{_{\omega}}(s')$, then ${\cal M}(s') > \omega$.
	    \label{sch:main:assump_b} 
	\end{enumerate}
	\noindent
	Then Algorithm \ref{sch:main}  %
	returns $\,s \in {\cal S}_{_{\!\!{\cal M}}}$ after a finite number of iterations.

\end{theorem}

\begin{proof}
Suppose that $\omega_0\in [0,\widehat{\cal M}]$. Then 
 the iterative scheme in each iteration finds $\,{s}_{_{\!\omega}} \!\in\! {\cal S}_{_{\!\!{\cal Q}_{_\omega}}}\!$ and specifies a new value for $\omega$ to be ${\cal M}({s}_{_{\!\omega}}).$  Condition\, 2
	is fulfilled trivially for $s' \!=\! {s}_{_{\!\omega}}$ since ${s}_{_{\!\omega}}$ maximizes ${\cal Q}_{_{\omega}}.$ Hence, $\,\omega < {\cal M}({s}_{_{\!\omega}})\,;$ {\em i.e.}, $\omega$ strictly increases from iteration to iteration while $\,\omega < \widehat{\cal M}\,.$ $\cal S$ is finite, hence, $\omega$ reaches $\widehat{\cal M}$ after a finite number of iterations. When that happens, $\,\widehat{s}_{\omega = \widehat{\cal M}} \in {\cal S}_{_{\!\!{\cal M}}}\,$ since, by condition  \ref{sch:main:assump_a}, ${\cal Q}_{\omega = \widehat{\cal M}}\,$ and ${\cal M}$ are co-optimal. The iterations stop when $\,\omega$ no longer increases. Hence, to prove the theorem, we show that $\,\omega\,$ does not change after it reaches $\widehat{\cal M}$.
	\begin{enumerate}
		\itemsep0pt%
		
		\item[$\pmb \implies$] Suppose that $\,\omega \!=\! \widehat{\cal M}\,.$ Since $\widehat{s}_{\omega = \widehat{\cal M}}$ maximizes both $\,{\cal M},\,{\cal Q}_{\omega = \widehat{\cal M}}\,\,,$ we have $\,\widehat{\cal M} = {\cal M}({s}_{_{\!\omega}}) \quad\Rightarrow\quad \omega = \widehat{\cal M} = {\cal M}({s}_{_{\!\omega}})\,.$
		
		\item[$\pmb \impliedby$] Conversely, suppose that $\,\omega \!=\! {\cal M}({s}_{_{\!\omega}})$, {\em i.e.}, $\omega$ does not change at line 5 of the algorithm. All values of $\,\omega\,$ specified in scheme \ref{sch:main} are values of $\cal M\,;$ hence, $\,\omega \!\leqslant\! \widehat{\cal M}\,.$ If $\,\omega \!<\! \widehat{\cal M}\,,$ then by condition \ref{sch:main:assump_b} $\,\omega \!<\! {\cal M}({s}_{_{\!\omega}})\,,$ which contradicts the current assumption. Hence, $\,\omega = \widehat{\cal M}\,.$
	\end{enumerate}
Suppose now that the condition required above,  $\omega_0\in [0,\widehat{\cal M}]$, is not satisfied ({\em i.e.}, $\omega_0 > \widehat{\cal M}$). Then, line 1 returns some  $\omega$ which must be lower than the maximal $\widehat{\cal M}$. Then the algorithm proceeds and reaches the optimum according to the proof above. 
\end{proof}

Given a quality measure $\,0 \!\leqslant\! {\cal M} \!\leqslant\! 1$ over $\cal S,$ we refer to a family $\{{\cal Q}_{_{\omega}}\} \,\omega \!\in\! [0,1]$ of measures over $\cal S,$ as a {\em family of auxiliary measures for} $\cal M$ if $\{{\cal Q}_{_{\omega}}\}$ contains at least one measure ${\cal Q}_{_{\!\omega'}}$ that is co-optimal with $\cal M$ over $\cal S,$ and there is some iterative process 
that finds ${\cal Q}_{_{\!\omega'}}$ from $\{{\cal Q}_{_{\omega}}\}.$ We refer to ${\cal Q}_{_{\!\omega'}}$ as a {\em co-optimal auxiliary measure}, and we refer to an algorithm that can optimize every member of $\{{\cal Q}_{_{\omega}}\}$ as an {\em auxiliary algorithm}.

In scheme \ref{sch:main}$,$ the  auxiliary algorithm is written in the most general form: $\argmax\,{\cal Q}_{_{\omega}}.$ 
In the next section, we provide a family of auxiliary measures and corresponding auxiliary algorithms, suitable for optimizing the Jaccard index, for different consistencies and constraints of the node set size.


\section{Optimizing the Jaccard index}
\label{sec:Jaccard}

After setting the framework and developing the necessary new optimization tool, we shall now turn to the main goal of this paper: Finding a tight upper bound on the obtainable Jaccard index.

\subsection{The Jaccard index}
\label{sec:ObjBased}

The Jaccard index (or the intersection over union measure) is a popular segmentation quality measure, applicable to a simple segmentation into two parts: foreground (or object) and background. 

Let $\,({\cal B}_{_{GT}}, {\cal F}_{_{GT}})\,$ and $\,({\cal B}_{s},{\cal F}_{s})\,$ be  two  foreground-background partitions corresponding to the ground-truth and a  segmentation $\,s \! \in \! \cal S\,$. 
The Jaccard index $J$ is given by:
\begin{equation}
	J(s)= \frac{|{\cal F}_{_{GT}}\cap {\cal F}_s|}{|{\cal F}_{_{GT}}\cup {\cal{F}}_s|}
\end{equation}

Given a hierarchy, we shall find, for each consistency and node subset size $|{\cal N}_s|$, the segmentation that maximizes the Jaccard index. 
This segmentation also maximizes the object-based
F-measure, as the two measures are equivalent~\cite{pont2016supervised}.

For two-part segmentation, only one segmentation is a-consistent with a BPT hierarchy: the two children of the root. We ignore this trivial case in the following discussion.

\subsection{Segmentation dimensions}\label{sec:SegDim}

Let ${\cal S}^{^{_2}} \!\!\subset\! \cal S$ be the subset of all possible 2-segment segmentations, consistent with the hierarchy. Denote the areas of the ground-truth parts by $B \! = \! |{\cal B}_{_{GT}}|\, , \,F \!=\! |{\cal F}_{_{GT}}|.$ Let $\,{\cal X}_s$ be one segment of a segmentation $\,s \! \in \! {\cal S}^{^{_2}}$. 
Considering this segment as foreground, denote its areas inside the ground-truth's parts by $(b_s \!=\! |{\cal X}_s \cap {\cal B}_{_{GT}}| \, , f_s \!=\! |{\cal X}_s \cap {\cal F}_{_{GT}}|).$ 
The Jaccard index is then
\begin{equation}\label{eq:Jacq}
	J(s) = \frac{|{\cal F}_{_{GT}}\cap {\cal X}_s|}{|{\cal F}_{_{GT}}\cup {\cal X}_s|} = \frac{f_s}{F + b_s} = \Psi(b_s\,,f_s).
\end{equation}

Alternatively, the foreground can be specified by the complementary segment $\,I \backslash {\cal X}_{s}\,$. The corresponding areas  inside the ground-truth's parts are $\,(B-b_s\,,F-f_s).$ The Jaccard index associated with this foreground is
\begin{equation}\label{eq:Jacq_comp}
	J^c(s) = \Psi(B - b_s\,,F - f_s) = \frac{F - f_s}{F + B - b_s}.
\end{equation}

Optimizing $J(s)$ for b-consistency provides a cut in tree  $\,{\cal N}_s$. Both $\,{\cal F}_s$ and $\,{\cal B}_s$ are unions of nodes of this cut. 
The c/d consistencies allow one segment to be specified as the complement of the other. The hierarchy may match better either ${\cal F}_{_{GT}}$ or ${\cal B}_{_{GT}}$. Thus, we optimize both $\,J(s)\,$ and $J^c(s)\,$ (for the same size of $\,{\cal N}_s$) and choose the better result.

The values $(b_s\,,f_s)$ are the main optimization variables. We refer to them as {\em segmentation dimensions}.

\subsection{Applying co-optimality for optimizing \texorpdfstring{$J(s)$}{J(s)}}\label{sec:J-index}

\subsubsection{Geometrical interpretation}

\begin{figure*}[tbh]
	\centering{\includegraphics[width=0.9\textwidth]{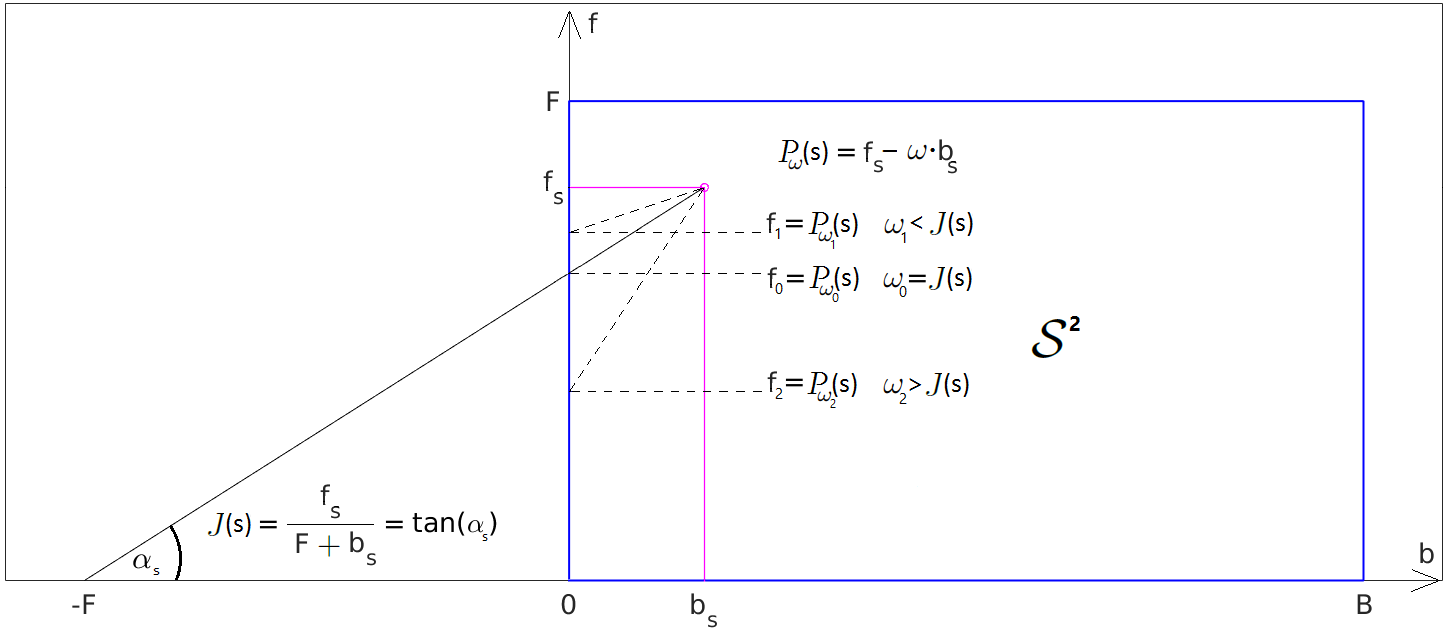}}
	\caption{A geometrical interpretation: the Jaccard index $J(s)$ is the tangent of the angle $\alpha_s$. }\label{fig:scheme_examp_1}
\end{figure*}

Our goal is to find 
\begin{equation}
   (\, \widehat{J} \, , \, \widehat{s} \, ) = \maxargmax \limits_{\!\!\! s \, \in \, {\cal S}^{^{_2}}} J(s).
\end{equation}
A key idea is to observe that the $J(s)$ value may be interpreted geometrically using the graph of segmentation dimensions $(b,f)$.
Selecting the segment $\,{\cal X}_s\,,$ every segmentation $\,s \!\in\! {\cal S}^{^{_2}}\,$ corresponds to a point $(b_s, f_s)$ inside the rectangle $(0,B)$\texttimes$(0,F).\,\,J(s)$ is $tan(\alpha_s),$ where $\alpha_s$ is the angle between the $\,b\,$ axis and the line connecting the point $(b_s, f_s)$ with the point (-$F, \,0$)\,; see Figure \ref{fig:scheme_examp_1}. The geometry implies that $tan(\alpha_s) \in [0,1]$, 
consistently with $tan(\alpha_s)=J(s) \in [0,1]$.

\subsubsection{A family of auxiliary measures}
\label{sec:cooptAuxJ}
For every $\omega \in [0, 1]$, let 
\begin{equation}
\label{auxuliary_meas}
    P_{_{\!\omega}}(s) = f_s - \omega \cdot b_s 
\end{equation}
be a measure over ${\cal S}^2$. 
Note that, geometrically, $P_{_{\!\omega}}(s)$ is the oblique projection at the $\arctan(\omega)$ angle of point $(b_s, f_s)$ onto the $f$ axis. 
The following two observations imply that $J(s)$ and the projection (at $\arctan(\widehat{J}\,)$ angle) $P_{_{\!\widehat J}}(s)$ are co-optimal measures. 

\begin{enumerate}%
	\itemsep2pt%
	
	\item $J(s)\,$ and $\,P_{_{\!\!J(s)\!}}(s)\,$ are equivalent measures.
	
	\item $\,P_{_{\!\!J(s)\!}}(s)\,$ and $\,P_{_{\!\widehat J}}(s)$ are co-optimal measures.
	
\end{enumerate}

The first observation is clear: ranking the elements of ${\cal S}^{^{_2}}$ by $J(s) = tan(\alpha_s)$ is  equivalent to ranking them by their projection at angle $\,\alpha_s\,,$ {\em i.e.}, by $P_{_{\!\!J(s)\!}}(s)$ ; see Figure \ref{fig:scheme_examp_1}.

The second observation states that there is a constant angle $\arctan(\omega)$ with $\omega = \widehat{J} $(not depending on $s$), such that  the projection $P_{_{\!\omega}}(s) $ at this angle and $P_{_{\!\!J(s)\!}}(s)$ are co-optimal. By the first observation, $P_{_{\!\!J(s)\!}}(s)$ is maximized by  $\widehat{s}$. Every non-optimal segmentation 
corresponds to a point below the line $[(-F,0)-(b_{\widehat s},f_{\widehat s})]$ and its constant angle projection satisfies $P_{_{\!\omega}}(s) < P_{_{\!\omega}}(\widehat{s})$. 
$P_{_{\!\omega}}(s)$ 
is maximized only by points
lying on this line, as is also the case with  $P_{_{\!\!J(s)\!}}(s)$.

Thanks to these two observations, the family $\{P_\omega\}$ is a family of auxiliary measures for the Jaccard index. 
The optimization process (Algorithm \ref{sch:main}) maximizes this auxiliary measure in every iteration: 
\begin{equation}
    (\, \widehat{P}_{_{\!\omega}} \, , \, \widehat{s}_{_{\!\omega}} \, ) = \maxargmax \limits_{\!\!\! s \, \in \, {\cal S}^{^{_2}}} P_{_{\!\omega}}(s)
\end{equation}

Note that $P_{\omega}(s)$ is linear in $(b_s,f_s)\,,$  which simplifies its maximization. 
To use scheme \ref{sch:main} to find $\widehat{s}$ (and $\,\widehat{J}\,),$ the second condition of Theorem \ref{sch:main:theorem}, which guarantees that $\omega$ strictly increases at every iteration while $\omega \!<\! \widehat{J}$, must be met as well.

Figure \ref{fig:scheme_examp_3} geometrically proves this property. Indeed, let $\omega \!\in\! [0,\widehat{J})$ and $\,\widehat{s},s' \!\in\! {\cal S}^{^{_2}}$ such that $\,P_{_{\!\omega}}(\widehat{s}) \!\leqslant\! P_{_{\!\omega}}(s').$ Observe that the angle $\,\alpha_{s'}$ must be larger than the projection angle of $\,P_{_{\!\omega}}\,,$ {\em i.e.}, $\arctan(\omega).$ The detailed proof is left to the reader. 
\begin{figure*}[h]
	\centering{\includegraphics[width=0.9\textwidth]{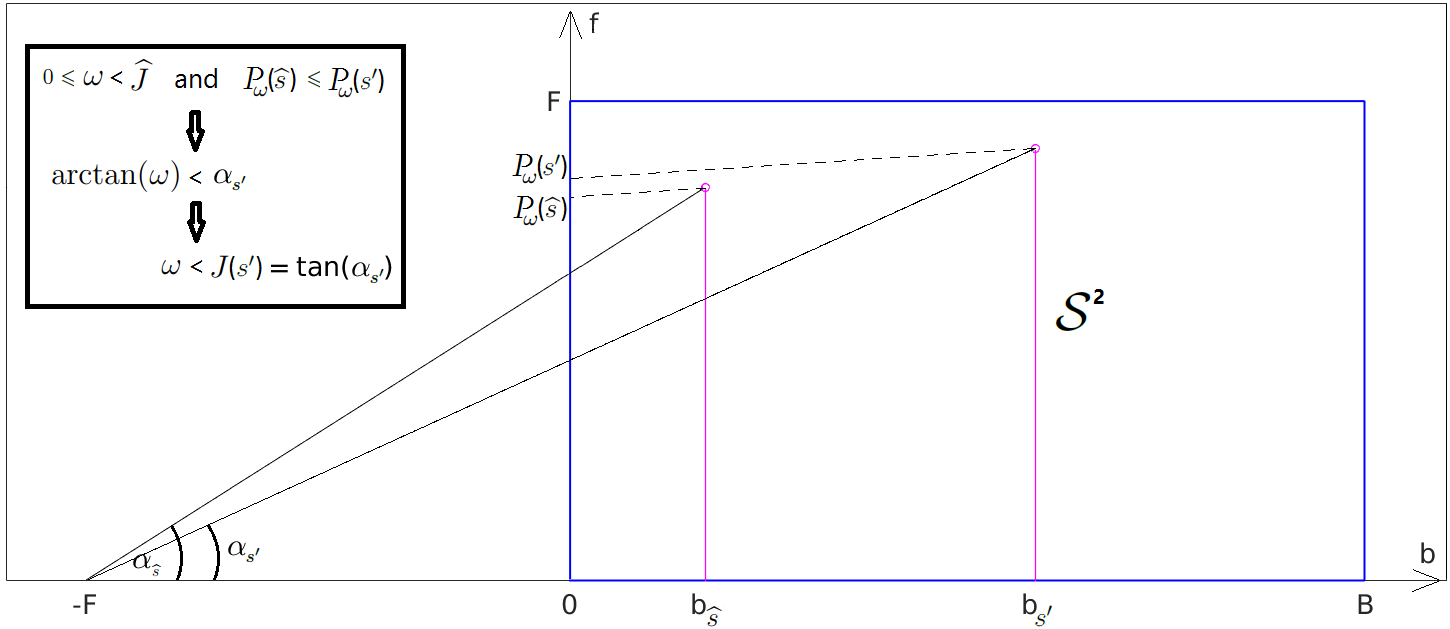}}
	\caption{An illustration showing that $\,\omega\,$ strictly increases from iteration to iteration, while $\omega < \widehat{J}.$
	}\label{fig:scheme_examp_3}
\end{figure*}
Therefore, by Theorem \ref{sch:main:theorem}  
\begin{theorem}\label{theorem:Jaux}
	
	For ${\cal M} \!=\! J$, $\{{\cal Q}_{_{\omega}}\} \!=\! \{P_{_{\!\omega}}\}$, and $\omega \!\in\! [0,1]$, scheme \ref{sch:main} (starting from $\omega_0\in[0,\widehat{J}]$) returns the best segmentation $\,\widehat s\,$ after a finite number of iterations.
\end{theorem}

\begin{remark}\label{remark:AuxFamaly}
	
	Optimizing $J^c(s)$  is similarly done.
\end{remark}

\subsection{Optimizing $J(s)$  for hierarchical segmentation }\label{sec:ob-m-hc}

Using scheme \ref{sch:main} reduces the optimization of $J(s)$ to iterations where  auxiliary measures are optimized. The auxiliary algorithm provides a foreground-background segmentation $s \!\in\! {\cal S}^{^{_2}}$ whose dimensions ($b_s,f_s$) maximize the auxiliary measure corresponding to the current iteration. 
In this work, we use the hierarchy for this optimization, and the auxiliary algorithm returns the best segmentation $\,s \!\in\! {\cal S}^{^{_2}}$  together with the corresponding subset $\,{\cal N}_s \!\subset\! \cal T,\,$ which both depend on the required consistency of $\,s$ with $\cal T.$

\subsubsection{Specifying ${\cal N}_s$ and the  segmentation $s$ for various consistencies}\label{sec:SegFrBk}

Here, we specify the relation between the hierarchy and the segmentation for each of the different consistencies considered in this paper. 

\noindent
{\bf a-consistency:} Trivial and not considered here, as discussed above.

\noindent
{\bf b-consistency:} ${\cal N}_s\,$ is a partition of $\,I.\,$ A segmentation $\,s\,$ is specified by assigning some nodes of $\,{\cal N}_s$ to the foreground $\,{\cal F}_s\,,$ and the rest to the background $\,{\cal B}_s.$

\noindent
{\bf c-consistency:} The nodes of $\,{\cal N}_s\,$ are disjoint, but their union is not the full image. 
The segments of $\,s\,$ are $\,\sbullet{\cal N}_s\,$ (the union  of the regions corresponding to the nodes in ${\cal N}_s$) and the complement $\,I \backslash \sbullet{\cal N}_s.$ 

\noindent
{\bf d-consistency:} Not all nodes of $\,{\cal N}_s\,$ are necessarily disjoint, and their union is not necessarily the full image. %
The segments of $\,s$ are specified as follows: 

Let $\,\cal N \!\subset\! T\,$ be a subset of nodes. Because the nodes belong to a hierarchy, each pair of nodes is either disjoint or nested.
Denote by $\,{\cal K}^{^{_0}}_{_{\!\cal N}} \!\subset\! \cal N$ the subset of disjoint nodes that are not nested in any other node from $\cal N.$ Recursively, denote by ${\cal K}^i_{_{\!\cal N}} \!\subset\! \cal N\,$ the subset of disjoint nodes that are not nested in any other node from ${\cal N} \backslash\{ \cup_{j = 0}^{i-1} \, {\cal K}^j_{_{\!\cal N}} \}.$ We refer to each $\,{\cal K}^i_{_{\!\cal N}}\,$ as a {\em layer of} $\,\cal N.\,$ Note that $\,\sbullet{\cal K}^i_{_{\!\cal N}} \! \subset \! \sbullet{\cal K}^j_{_{\!\cal N}} \,\,\, \forall \, i \! > \! j\,$ (each subsequent layer is nested in any previous layer); hence, $\,\sbullet{\cal K}^{^{_0}}_{_{\!\cal N}} = \sbullet{\cal N}.\,$ 
Let $\,i^{^{_N}}_{_{\cal N}},$ be the index of the layer in which the node $N$ lies. 
Note that the set of layers is a partition of $\,\cal N,$ {\em i.e.}, every node $\,N \! \in \! \cal N$ is associated with exactly one index $\,i^{^{_N}}_{_{\cal N}}.$ Note that $\,i^{^{_N}}_{_{\cal N}}$ is the number of nodes in $\cal N$ in which $N$ is nested. Let $i_{_{\cal N}}^{max}$ be the largest index corresponding to a nonempty layer. The segmentation is specified from 
\begin{multline}
\label{def:D_N}
	{\cal D}_{_{^{\!{\cal N}}}} = \bigg{\{}\,D^i_{_{^{\!\cal N}}} \enskip|\enskip D^i_{_{^{\!\cal N}}} = \sbullet{\cal K}^{2 \cdot i}_{_{^{\!\cal N}}} \backslash \sbullet{\cal K}^{2 \cdot i + 1}_{_{^{\!\cal N}}}\,,\,\,\\
	0 \leqslant i \leqslant \floor*{\frac{i_{_{\cal N}}^{max}}{2}} 
	\,\bigg{\}}
\end{multline}

Each $D^i_{_{\!\cal N}}$ is the set-difference of the layers $\,2 \!\cdot\! i\,$ and $\,2 \!\cdot\! i + 1.$ Since each subsequent layer is nested in any previous layer, all $D^i_{_{\!\cal N}}$ are disjoint. The segments of $s$ are $\sbullet{\cal D}_{_{\!{\cal N}_s}}$ and the complement $I \backslash \sbullet{\cal D}_{_{\!{\cal N}_s}};$ see Figure \ref{fig:D_N_example}.

\begin{figure*}[t]
 	\def\svgwidth{1\textwidth}
	\centering{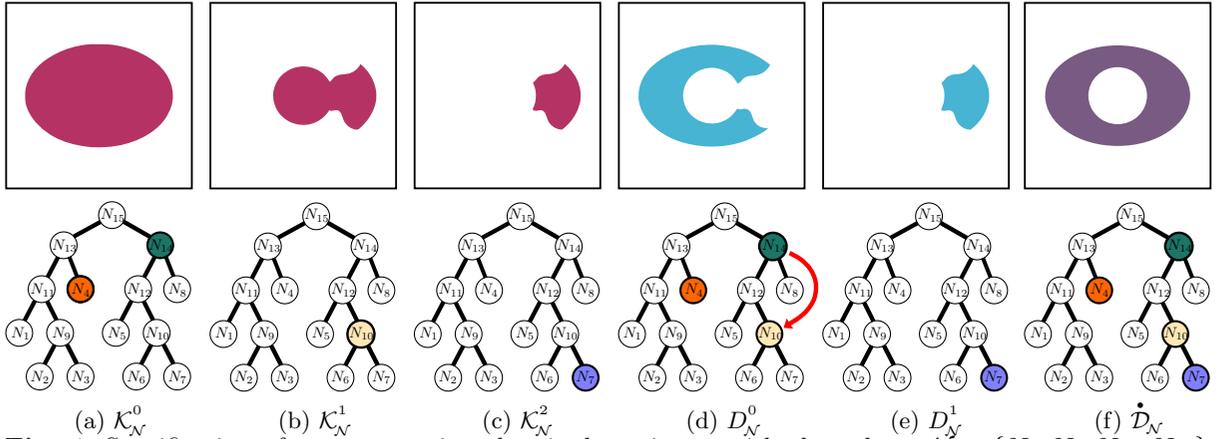}
	\caption{Specification of a segmentation that is d-consistent with the subset $\,{\cal N} \! = \! \{N_4, N_7, N_{10}, N_{14}\}\,$ consisting of nodes from the hierarchy described in Figure \ref{fig:hierarchy-examp}. The segments are $\sbullet{\cal D}_{_{\!\cal N}}$ and the complement $I \backslash \sbullet{\cal D}_{_{\!\cal N}}.$ Here maximal non-empty layer correspond to the index $i_{_{\cal N}}^{max} \! = 2.\,$ The layers are: {\bf (a)} ${\cal K}^{^{_0}}_{_{\!\cal N}} \! = \! \{N_4, N_{14}\}\,$ {\bf (b)} ${\cal K}^{^{_1}}_{_{\!\cal N}} \! = \! \{N_{10}\}\,$ {\bf (c)} ${\cal K}^{^{_2}}_{_{\!\cal N}} \! = \! \{N_7\}.\,$ The set differences between subsequent layers are {\bf (d)} $D^{^{_0}}_{_{\!\cal N}} \! = \! {\cal K}^{^{_0}}_{_{\!\cal N}} \! \backslash {\cal K}^{^{_1}}_{_{\!\cal N}}\,$  {\bf (e)} $D^{^{_1}}_{_{\!\cal N}} \! = \! {\cal K}^{^{_2}}_{_{\!\cal N}} \! \backslash \varnothing.\,$ The final segmentation is specified by {\bf (f)} $\sbullet{\cal D}_{_{\!\cal N}} \! = \! D^{^{_0}}_{_{\!\cal N}} \! \cup \! D^{^{_1}}_{_{\!\cal N}}$; see Section \ref{sec:SegFrBk}.
	}\label{fig:D_N_example}
\end{figure*}

\subsubsection{Calculation of segmentation dimensions for various consistencies}\label{sec:NodeDim}

To calculate the  segmentation dimensions $(b_s,f_s)$  efficiently, we use a tree data structure to represent the tree hierarchy. For each node $N$ of the tree, we store the area of the node inside the ground truth's parts $(b^{^{_N}} \!\! = \!\! |N \cap {\cal B}_{_{GT}} |\, , \,f^{^{_N}} \!\! = \!\! |N \cap {\cal F}_{_{GT}}|).$ Similarly to the {\em segmentation dimensions}, specified above, we refer to these values as {\em node dimensions}. Note that the dimensions of a union of disjoint nodes are equal to the sum of the dimensions of all nodes from the union, and the dimensions of a set-difference of two nested nodes are equal to the difference of their dimensions.
Given a segmentation $s=({\cal X}_s\,,\,I \backslash {\cal X}_{s})  \!\in\! {\cal S}^{^{_2}},$ the calculation of its dimensions $(b_s,f_s)$ (which are the dimensions of the segment ${\cal X}_s$) from the dimensions of the nodes of a subset ${\cal N}_s$ depends on the required consistency of $s$ with ${\cal N}_s\,$:

\smallskip{\bf b-consistency:} $({\cal X}_s \!=\! {\cal F}_s).$ $(b_s,f_s)$ are calculated by the sum of the dimensions of the nodes assigned to $\,{\cal F}_s.$

\smallskip{\bf c-consistency:} $({\cal X}_s \!=\! \sbullet{\cal N}_s).$ $(b_s,f_s)$ are calculated by the sum of the dimensions of ${\cal N}_s.$

\smallskip{\bf d-consistency:} $({\cal X}_s \!=\! \sbullet{\cal D}_{_{\!{\cal N}_s}}\!).$ 
By the observations above about the sums and difference of dimensions and Equation \eqref{def:D_N}, the dimensions $(b_s,f_s)$ are calculated by the sum of the dimensions of all nodes from ${\cal N}_s\,,$ each multiplied by an appropriate sign: -1 to the power of $\,i^{^{_N}}_{_{{\cal N}_s}}\!.$ 
More formally, we can write $(b_s,f_s)$ as the  expression below. 
Note that for the b/c consistencies, ${\cal N}_s$ consists of a single layer: $i^{^{_N}}_{_{{\cal N}_s}} \! \!=\! 0\enskip\forall N \!\in\! {\cal N}_s$. Therefore, this expression is valid for all (b/c/d) consistencies.

\smallskip{\bf A unified expression of segmentation dimensions:}  
\begin{subequations}
\label{def:dim}
\begin{align}
\label{def:dim1}
	b_s = &\bigg{(} \!\!\!\!\!\!\!\!\!\! \sum \limits_{\substack{\,\,\,\,\,\,\,\,\, N \in\, {\cal N}_s:\\\,\,\,\,{\cal N}_s \,\text{specifies}\,\,{\cal X}_s}} \!\!\!\!\!\!\!\!\! b^{^{_N}} \! \cdot (-1)^{^{_{i^{^N}_{{\cal N}_s}}}}\bigg{)} \\
	\label{def:dim2}
	f_s = & \bigg{(} \!\!\!\!\!\!\!\!\!\! \sum \limits_{\substack{\,\,\,\,\,\,\,\,\, N \in\, {\cal N}_s:\\\,\,\,\,{\cal N}_s \,\text{specifies}\,\,{\cal X}_s}} \!\!\!\!\!\!\!\!\! f^{^{_N}} \! \cdot (-1)^{^{_{i^{^N}_{{\cal N}_s}}}}\bigg{)}
\end{align}
\end{subequations}
\begin{remark}\label{remark:WellDefSegDim}

Since for a segmentation $s \!\in\! {\cal S}^{^{_2}}$ the subset $\,{\cal N}_s$ is not necessarily unique, we could ask whether the expression \eqref{def:dim} is well-defined, {\em i.e.}, whether we get the same area ($b_s \! = \! |{\cal X}_s \cap {\cal B}_{_{GT}}| \, , f_s \! = \! |{\cal X}_s \cap {\cal F}_{_{GT}}|$) for different subsets ${\cal N}_s.$ The answer to this question is positive, due to the properties of node dimensions for the union of disjoint nodes and for the set-difference of nested nodes.
\end{remark}

\subsubsection{Auxiliary measures additivity}\label{sec:Decom}

A particularly useful property of the auxiliary measures is their additivity. 
Consider some  attribute defined on every node in the tree. If the  attribute of each non-leaf node is the sum of the attributes of the node's children, then we say that this attribute is  {\em additive}. 

For a specific projection $P_{_{\!\omega}}$, the two dimensions of a node may be merged into one attribute, $ A^{\!^{P_{_{\!\omega}}}}\!(N) = f^{^{_N}} \! - \omega \cdot b^{^{_N}}$. By inserting  \eqref{def:dim1} and \eqref{def:dim2}
into $P_{_{\!\omega}}(s) = f_s \, - \, \omega \cdot b_s $ \eqref{auxuliary_meas}, we get a closed form, simplified linear expression for the auxiliary measure of the segmentation $s$. We may refer to this measure, alternatively, as the {\em benefit} of the corresponding node set ${\cal N}_s$:

\begin{equation}\label{eq:Pj_attr}
	 P_{_{\!\omega}}(s) =
	 B[{\cal N}_s]=\!\!\!\!\!\!\!\!\! \sum \limits_{\substack{\,\,\,\,\,\,\,\,\, N \in\, {\cal N}_s:\\\,\,\,\,{\cal N}_s \,\text{specifies}\,\,{\cal X}_s}} \!\!\!\!\!\!\!\!\! A^{\!^{P_{_{\!\omega}}}}\!(N) \cdot (-1)^{^{_{i^{^N}_{{\cal N}_s}}}}  
\end{equation}

Note that each non-leaf node $N$ is the union of its disjoint children; hence, the dimensions $(b^{^{_N}}, f^{^{_N}})$ are the sum of the dimensions of the children of $N,$ which implies the additivity for $A^{\!^{P_{_{\!\omega}}}}\!(N).$ The additivity property holds for all projections. For simplification, we refer to the attribute of $N$ as $A(N)$.
 
The auxiliary algorithms search for the subset of nodes maximizing the benefit~\eqref{eq:Pj_attr}. 
These optimization tasks are performed under the constraint: $|{\cal N}_s| \,\leqslant k$. 

While \eqref{eq:Pj_attr} provides a general expression for all consistencies, in practice we use the following consistency-dependent expressions, which are equivalent and more explicit.

\begin{property} (Equivalent benefit expressions)
\label{def:benefit}

\begin{description}
    \item[b-consistency:] ${\cal N}$ is a partition of $I$ and\\
    $B[{\cal N}] = \sum \limits_{\substack{\,\,\,\,\,\,\,\,\, N \in\,{\cal N}:\\\,\,\,\,\,\,\,\,A(N)\,>\,0}} \!\!\!\!\!\!\! A(N)$
    \item [c-consistency:] ${\cal N}$ consists of disjoint nodes and 
    $B[{\cal N}] = \sum \limits_{\,\,\,\,\,\,\, N \in\,\cal N} \!\!\!\! A(N)
	$
	\item [d-consistency:] $ B[{\cal N}] = \sum \limits_{\,\,\,\,\,\,\, N \in\,\cal N} \!\!\!\! A(N) \cdot (-1)^{^{_{i^{^N}_{_{\cal N}}}}}$
\end{description}
\end{property}
Here and below, we prefer to use the more general 
${\cal N}$ (over ${\cal N}_s$), when the discussion applies to general  sets of nodes from the tree. 

The proposed auxiliary algorithms (described below) are not restricted to the auxiliary measures discussed above; they would work for any additive measure $\cal Q.$ The additivity is crucial, because otherwise the score ${\cal Q}(s)$ is ill-defined, {\em i.e.}, it may result in different score values for different subsets ${\cal N}_s$ specifying the same $s \!\in\! {\cal S}^{^{_2}}\!.$

\subsubsection{Using the tree structure for maximizing the auxiliary measures}\label{sec:TreeStruct}

The maximization of the benefit (property \ref{def:benefit}) results in a subset of nodes subject to the consistency constraints, with the maximal benefit in $\,\cal T.$ The key observation in this maximization is that a subset with the maximal benefit in a subtree ${\cal T}^N\!$ can be obtained from subsets with the maximal benefit in the subtrees of children of $N.$
That is, we can use the recursive structure of the tree $\cal T$ to maximize the benefit. 

Let $\,\cal N' \!\subset\! N \!\subset\! T.$ We say that $\,\cal N'$ is {\em best} if it has the highest benefit relative to every other subset of $\cal N$ with the same number of nodes. Depending on the context, $\,\cal N'$  should also have the properties associated with the consistency; {\em i.e.}, being a partition (for b-consistency) or belong to a single layer (c-consistency). Interestingly, we also need the notion of worst subsets. $\,\cal N'$ is {\em worst} if it has the minimal benefit relative to other subsets of $\,\cal N$ of the same size.

\begin{remark}\label{remark:BestSubsetsWithSameBenefitAndDifferentSizes}
	
	Note that within the same consistency type, there can be several best/worst subsets in $\,\cal N,$ having the same benefit but not necessarily of the same size.
\end{remark}

Thus, a subset $\cal N$ maximizes the benefit (property  \ref{def:benefit}), if and only if $\cal N$ is a best subset in $\cal T.$ Below, by referring to $\cal N$ as best without specifying in which subset of $\cal T$ the subset $\cal N$ is best, we mean that $\cal N$ is best in the entire $\cal T.$

The following claim %
readily follows from the additivity properties of the dimensions  (Section \ref{sec:Decom}). %
\begin{lemma}\label{lemma:Psubset}
	\begin{enumerate}
		\itemsep3pt%
		
		\item[(a)] \!\!Let ${\cal N}_1\!$ and ${\cal N}_2$ be subsets of nodes, such that $\sbullet{\cal N}_1\!$ and $\sbullet{\cal N}_2$ are disjoint, then: $B[{\cal N}_1 \cup {\cal N}_2] \!=\! B[{\cal N}_1] + B[{\cal N}_2]$
		
		\item[(b)] \!\!Let $N$ be a node and $\,\cal N$ be a subset of nodes, such that $\sbullet{\cal N} \!\subset\! \,N$ then: $B[{\{N\} \cup \cal N}] \!=\! A(N) - B[\cal N]$
	\end{enumerate}
\end{lemma}

Lemma \ref{lemma:Psubset}(b) applies only to the d-consistency, in the case where $\sbullet{\cal N}$ and $N$ are not disjoint. The set of nodes ${\{N\} \cup \cal N}$ corresponds to a segment that is the set difference between $N$ and the segment specified by  ${\cal N}$, which leads to the claim on the benefit. 

The children of a non-leaf node are disjoint and nested in the node, which implies the following claim.%
\begin{lemma}\label{lemma:best-worst}
	
	Let $N \!\in\! \cal T$ be a non-leaf node: $\!N \!=\! N_r \cup N_l\,,$ where $N_r\!$ (right), $N_l\!$ (left) are its children. Let $\,{\cal N}^N\!$ be a subset of $\,{\cal T}^N\!$ and let ${\cal N}^{N_r}\!,\,{\cal N}^{N_l}\!$ be (possibly empty) subsets of $\,{\cal N}^N\!$ from $\,{\cal T}^{N_r}\!$ and ${\cal T}^{N_l}\!$ respectively. Then:
\end{lemma}
\begin{enumerate}%
	\itemsep2pt%
	
	\item[(a)] If $\,N \notin {\cal N}^N$ then: $\quad {\cal N}^N$ is best/worst in $\,{\cal T}^N \quad\enskip \Rightarrow \quad\enskip {\cal N}^{N_r}\!$, ${\cal N}^{N_l}$ are best/worst in $\,{\cal T}^{N_r}\!$, ${\cal T}^{N_l}$
	
	\item[(b)] If $\,N \in {\cal N}^N$ then: $\quad {\cal N}^N$ is best/worst in $\,{\cal T}^N \quad\enskip \Rightarrow \quad\enskip {\cal N}^{N_r}\!$, ${\cal N}^{N_l}$ are worst/best in $\,{\cal T}^{N_r}\!$, ${\cal T}^{N_l}$
	
\end{enumerate}
\begin{proof}
	Assume the opposite about any $\,{\cal N}^{N_r},\,{\cal N}^{N_l}.$  Lemma \ref{lemma:Psubset} implies that $\,{\cal N}^N\!$ can be improved/worsened, which contradicts $\,{\cal N}^N\!$ being best/worst.
\end{proof}

Lemma \ref{lemma:best-worst} specifies the necessary condition for a best/worst subset in $\,{\cal T}^N\!.$ With its help, the search for the best subsets in $\,{\cal T}^N\!$ can be significantly reduced, making this search feasible. Namely, for finding a best subset in $\,{\cal T}^N\!,$ it is enough to examine only those subsets that are best/worst in the subtrees of the children of $\,N.$ The following (trivial) sufficient condition for best/worst subset in $\,{\cal T}^N\!$ is to examine all possible candidates.

\begin{lemma}\label{lemma:conclBest}
	 Let $N \!\in\! \cal T$ be a non-leaf node. The subset ${\cal N} \!\subset\! {\cal T}^N\!$ having the largest/smallest benefit from the following is a best/worst subset in ${\cal T}^N\!:$
	\begin{enumerate}%
		\itemsep0pt%
		
		\item[(a)] The union of best/worst subsets in $\,{\cal T}^{N_r}\!$ and in ${\cal T}^{N_l}\!,$ having the maximal/minimal benefit among all such unions of size $|{\cal N}|.$
		
		\item[(b)] $N$ itself and the union of worst/best subsets in $\,{\cal T}^{N_r}\!$ and in ${\cal T}^{N_l}\!,$ having the minimal/maximal benefit among all such unions of size $|{\cal N}| \,-\, 1.$
		
	\end{enumerate}
	
\end{lemma}

We can now describe the auxiliary algorithms. From a high-level point of view, they work as follows: 
At the outset of the run, each auxiliary algorithm specifies each leaf of $\,\cal T$ as both the best and the worst subset (of size 1)  in the trivial subtree of the leaf. Then, each auxiliary algorithm visits all non-leaf nodes of $\,\cal T$ once, in a post-order of tree traversal which guarantees visiting every node after visiting its children. When visiting a node $N,$ each algorithm finds the best/worst subsets in ${\cal T}^N\!$ using Lemma \ref{lemma:conclBest}.

\subsection{The auxiliary algorithms}

\subsubsection{Preliminaries}\label{sec:auxAlgs}

Generally, the algorithm works as follows: Starting from the hierarchy leaves, the algorithms calculates the maximal auxiliary quality measure for every node and for every budget (up to $k$) in its subtree. When reaching the root, the decision about the particular nodes used for the optimal auxiliary measure is already encoded in the hierarchy nodes and is then explicitly extracted. Like \cite{pont2012upper}, it is a dynamic programming algorithm.  

The following variables and notations are used within the algorithms:
\begin{enumerate}%
	\itemsep2.64pt%
	
	\item $N_1$, $N_2$, $N_3, \dots, N_{|{\cal{T}}|}$ is the set of all nodes, ordered in a post-order of tree traversal.
	
	\item $N_l$ (left) and $N_r$ (right) are the children of a non-leaf node $N.$
	
	\item $A(N)$ is an additive attribute of a node $N.$ Recall that $A(N) = A(N_r) + A(N_l)$.
	
	\item $k \! \in \! \mathbb{N}$ is a constraint specifying the maximal size of the best subset $ (\,1 \!\leqslant\! k \!\leqslant\! |\cal L|\,)$. 
	
	\item $t(N) = \min(\,k\,,\,|{\cal{L}}^N|\,)\,$ is the maximal allowed size of a best/worst subset in $\,{\cal T}^N\!,$ which is limited by $k$ or by the number of leaves in $\,{\cal T}^N\!.$

 \item $r$ is the number of nodes in the node subset associated with the right child of $N$. It depends on $N$ and is optimized by the algorithms. The range of $r$ values is  denoted $(\, r_{min} \, , \, r_{max} \,)$.
	\item ${\cal H}_+^N\!(i)\,/\,{\cal H}_-^N\!(i) \enskip\,\, i \!=\! 1, \dots, t(N)\,$ are best/worst subsets of size $\,i\,$ in ${\cal T}^N\!.$ The best subset ${\cal H}_+^{Root}[k]$, denoted $\cal H,$ is the output of the auxiliary algorithm, maximizing the benefit; see however remark \ref{remark:MinimalBestSubset}. These subsets are used to describe the algorithm, but are not
 variables of the algorithm. 
	
	\item $B_+^N[i]\,/\,B_-^N[i] \quad i \!=\! 1 \dots t(N)\,$ are vector variables stored in node $N,$ holding the benefits of ${\cal H}_+^N\!(i)\,/\,{\cal H}_-^N\!(i).$
	
	\item $R_+^N[i]\,/\,R_-^N[i] \quad i \!=\! 1 \dots t(N)\,$ are vector variables stored in node $N,$ holding the number of those nodes in ${\cal H}_+^N\!(i)\,/\,{\cal H}_-^N\!(i)\,,$ which belong to $\,{\cal T}^{N_r}$ (the subtree of $N_r$).
	The number of nodes in $\,{\cal T}^{N_l}$  follows.
	\item $Q\,$ is queue data structure, used to obtain $\cal H$ from the vectors $R_+^N\,/\,R_-^N$.
\end{enumerate}

To find the best subset consisting of a single layer (as is the case for b/c consistency), we need to examine only the corresponding best subsets and disregard the worst subsets. 
In this case, we simplify the notation and use  $B^N,R^N,{\cal H}^N\!$  instead of $B_+^N,R_+^N,{\cal H}_+^N.$

\begin{remark}
\label{remark:CooptimalitySubsets}
	
Different optimal subsets for different $k$ are associated with different $\omega$ parameter values. Therefore, the set of subsets  $\{{\cal H}^{Root}(i) ; \ i < k\}$ obtained with the best subset ${\cal H}^{Root}(k)$ are not optimal as well. 
\end{remark}

\subsubsection{b-consistency}
\label{sec:auxAlgB}

The auxiliary algorithm for b-consistency is formally given in Algorithm \ref{alg:b1}.

A best subset (for b-consistency) (def. \ref{def:benefit}), ${\cal H}^N\!(i)$, must be a $\cal T$-partition of $\,\sbullet{\cal T}^N \!\!=\! N.$ Hence, $N \!\in\! {\cal H}^N\!(i),$ if and only if $i \!=\! 1.$ 
Thus, $B^N[1]$ is the benefit of the node $N$ itself. To calculate $B^N[i]\,$ for $\,i \!>\! 1\,,$ we need only the condition (a) of Lemma~\ref{lemma:conclBest}\,, which implies that $B^N[i]$ is the maximum of $\,B^{N_r}[r] + B^{N_l}[i-r]\,,$ over all possible values of $\,r.$ This part is carried out in lines 1-5 of Algorithm \ref{alg:b1}.

The best subset ${\cal H} = {\cal H}^{Root}[k]$ and its subset of nodes with a positive attribute, denoted $\cal G$, are specified from the vectors $R^N\!.$ The number of nodes in $\cal H$ that belong to the subtree of the root's right child is $R^{Root}[k]$ (recall that $\,t(Root) \!=\! k),$ and their number in the subtree of the left child is $\,k- R^{Root}[k].$ The same consideration is applied recursively to every node $N,$ stopping when $R^{N}[i]$ is equal to zero. This part is carried out in lines 6-16 of Algorithm \ref{alg:b1}.

\begin{algorithm2e*}
	\caption{Auxiliary algorithm for b-consistency}
		\label{alg:b1}
	\DontPrintSemicolon
	\setcounter{AlgoLine}{0}
  
	\For(\tcp*[f]{See note \ref{itm:Hb}, Sec.~\ref{sec:auxAlgB}}){$N = N_1, \, N_2, \, N_3, \, \dots, \, N_{|{\cal{T}}|}$}{
		$(\, B^N \, , \, R^N \,)[1] = (\, \max\!\Big{(}A(N)\,,\,0\Big{)} \,\, , \,\, 0 \,)$\;
  
		\For(\tcp*[f]{See note \ref{itm:rangeb}, Sec.~\ref{sec:auxAlgB}}){$i = 2, \dots, t(N)$}{
			$(\, r_{min} \, , \, r_{max} \,) \,=\, (\, \max( \, 1 \, , \, i-t(N_l) \,) \,\, , \, \min(\, t(N_r) \, , \, i-1 \,) \,) \enskip $\;
   
			$(\, B^N \, , \, R^N \,)[i] \,=\, \maxargmax \limits_{r_{min} \,\, \leqslant \,\, %
				r \,\, \leqslant \,\, r_{max}} (\,B^{N_r}[r] + B^{N_l}[i-r] \,)$\;
		}
	}
	$(\, \cal H, \, G \,) \, = \, (\, \varnothing, \, \varnothing \,)$\;
 
	$Q${\LARGE .}$Enqueue(\, Root \, , \, t(Root) \,)$\;
 
	\Do{$\, Q \,\, is \,\, not \,\, empty$}{
		$(\, N \, , \, i \,) \, = \, Q${\LARGE .}$Dequeue()$\;
  
		\uIf{$(\, R^N[i] \, == \, 0 \,)$}{
			\lIf{$(\, A(N) > 0 \,)$}{
				${\cal G} \, \leftarrow \, N$
			}
			${\cal H} \, \leftarrow \, N$
		}
		\Else{
			$Q${\LARGE .}$Enqueue(\, N_r \, , \, R^N[i] \,)$\;
   
			$Q${\LARGE .}$Enqueue(\, N_l \, , \, i - R^N[i] \,)$
		}
	}
	\Return $(\, \cal H, \, G \,)$
	
\end{algorithm2e*}

\smallskip\noindent{\bf Notes:}
\begin{enumerate}
	\itemsep0pt%
	
	\item \label{itm:rangeb} The range $(\, r_{min} \, , \, r_{max} \,)$ is calculated as follows: $i$ is the number of nodes in the subset associated with $N$. 
 The number of nodes, $r$, associated with the right child should satisfy $1 \leqslant r \leqslant t(N_r)$. {(Note that the lower limit is 1 and not 0, because for b-consistency $N_s$ is a cut.)} The number of nodes $i-r$ associated with the left child should satisfy  
	$1 \leqslant i-r \leqslant t(N_l)$, which imples that $r$ should satisfy $i-t(N_l) \leqslant r \leqslant i-1$. 
 Therefore $r$ should be in the range $(\, r_{min} \, , \, r_{max} \,) = (\, \max( \, 1 \, , \, i-t(N_l) \,) \,\, , \, \min(\, t(N_r) \, , \, i-1 \,) \,)$.

	\item \label{itm:Hb}
	\begin{empty_env}
		\label{alg:b1_note2}
	\end{empty_env}
	$\cal H$ is a cut of $\,\cal T$ (Section \ref{sec:introduction:hierarchies}), which implies that the deepest node is no deeper than $|{\cal H}| - 1$. Hence, Algorithm \ref{alg:b1} can be accelerated by processing only those nodes whose depth is less than $k$. 
	
\end{enumerate}

\subsubsection{c-consistency}
\label{sec:auxAlgC}

A similar auxiliary algorithm for c-consistency is given in Algorithm \ref{alg:c1}.

Note that a c-best subset ${\cal H}^N\!(i)$ consists of disjoint nodes, but their union is not necessarily $N.$ %
For example, for ${\cal H}^N\!(1)\,,$ there are three possibilities: the best node from ${\cal T}^{N_r}\!,$ the best node from ${\cal T}^{N_l}\!,$ and $N$ itself, which are marked by {\em ad-hoc} values of 1,\,0, and -1, respectively, in $R^N[1].$

\begin{algorithm2e*}
	\caption{Auxiliary algorithm for c-consistency}\label{alg:c1}
	\DontPrintSemicolon
	\setcounter{AlgoLine}{0}
	\For{$N = N_1, \, N_2, \, N_3, \, \dots, \, N_{|{\cal{T}}|}$}{
		$B^N[0] = 0$\;
  
		\If(\tcp*[f]{Skip the leaves})
		{$(\, N \,\, is \,\, not \,\, a \,\, leaf \,)$}{
			\For(\tcp*[f]{See note \ref{itm:rangec}, Sec.~\ref{sec:auxAlgC}}){$i = 1, \dots, t(N)$}{
				$(\, r_{min} \, , \, r_{max} \,) \,=\, (\, \max( \,0 \, , \, i-t(N_l) \,) \,\, , \,\, \min(\, t(N_r) \, , \, i \,) \,) \enskip $\;
    
				$(\, B^N \, , \, R^N \,)[i] \,=\, \maxargmax \limits_{r_{min} \,\, \leqslant \,\, %
					r \,\, \leqslant \,\, r_{max}} (\, B^{N_r}[r] + B^{N_l}[i-r] \,)$
			}
		}
		\If%
		{$(\, N \,\, is \,\, a \,\, leaf \quad \mathbf{or} \quad A(N) \, \geqslant \, B^N[1] \,)$}{
		    \tcp{If $N\!$ is\! not a\! leaf\!, $\!\!B^N\![1]$\, is already initialized}
			$(\, B^N \, , \, R^N \,)[1] \,= (\, A(N) \, , \, -1 \,)$
		}
	}
	${\cal H} \, = \, \varnothing$\; %
	$Q${\LARGE .}$Enqueue(\, Root \, , \, t(Root) \,)$\;  %
	\Do%
	{$\, Q \,\, is \,\, not \,\, empty$}{
		$(\, N \, , \, i \,) \, = \, Q${\LARGE .}$Dequeue()$\;
  
		\lIf{$(\, R^N[i] \, == \, -1 \,)$}{
			${\cal H} \, \leftarrow \, N$
		}
		\Else{
			\lIf{$(\, R^N[i] \, > \, 0 \,)$}{
				$Q${\LARGE .}$Enqueue(\, N_r \, , \, R^N[i] \,)$
			}
			\lIf{$(\, R^N[i] \, < \, i \,)$}{
				$Q${\LARGE .}$Enqueue(\, N_l \, , \, i - R^N[i] \,)$
			}
		}
	}
	\Return $\cal H$
	
\end{algorithm2e*}

\noindent{\bf Notes:}
\begin{enumerate}
	\itemsep0pt%
	
	\item \label{itm:rangec} The calculation of $(\, r_{min} \, , \, r_{max} \,)$ is as above, but the ranges of both children start from 0.
	
	\item For coding convenience, we added a cell $B^N[0],$ which always takes the value 0.
	\item \label{itm:algoc} The algorithm should preferably select only nodes with a positive attribute. If the number of nodes with a positive attribute (in one layer) is less than $k$, then nodes with a non-positive attribute are selected as well. In this case, however, there is a subset with fewer nodes and with a bigger benefit, which can be specified from $(B^{Root}, \, R^{Root})\,;$ see Remark \ref{remark:MinimalBestSubset}.
\end{enumerate}

\subsubsection{d-consistency}
\label{sec:auxAlgD}

The auxiliary algorithm for d-consistency is formally given in Algorithm \ref{alg:d1}.

Unlike the other consistencies, a d-best subset may contain nested nodes, which requires additional variables. Unlike the algorithms for the other consistencies, here we use all the vectors $B_+^N$, $B_-^N$, $R_+^N$,  and $R_-^N$, including the additional cells $B_+^N[0],$ $B_-^N[0],$ $R_+^N[0]$ and $R_-^N[0]$ which always take the value $0.$ By Lemma \ref{lemma:conclBest}, and using the notations introduced in section \ref{sec:auxAlgs}, ${\cal H}_+^N\!(i)$ is the subset having the maximal benefit from $C_i^{1+} \!\!=\! {\cal H}_+^{N_r}\!(r) \cup {\cal H}_+^{N_l}\!(i\!-\!r)$ and $C_i^{2+} \!\!=\! \{N\} \cup {\cal H}_-^{N_r}\!(r) \cup {\cal H}_-^{N_l}\!(i\!-\!1\!-\!r),$ over all possible values of $\,r.$ 

For getting ${\cal H}_-^N\!(i),$ the subset having the minimal benefit, we use similar expressions; see Algorithm~\ref{alg:d1}. 
It may happen that the calculation of $B_+^N[i]$ using $B_-^N[i-1]$  includes $A(N)$ twice. To avoid this problem, we calculate $B_+^N[i]\,,B_-^N[i]$ in two passes through all values of $\,i$, the second pass being in decreasing order of $\,i$ (see lines 6-18 in Algorithm.~\ref{alg:d1}).

\begin{algorithm2e*}
	\caption{Auxiliary algorithm for d-consistency}\label{alg:d1}%
	\DontPrintSemicolon
    \setcounter{AlgoLine}{0}
	
	\For{$N = N_1, \, N_2, \, N_3, \, \dots, \, N_{|{\cal{T}}|}$}{
		$(\, B_+^N \, , \, B_-^N \, , \, R_+^N \, , \, R_-^N \,)[0] = (\, 0 \, , \, 0 \, , \, 0 \, , \, 0 \,)$\;
  
		\uIf{$(\, N \,\, is \,\, a \,\, leaf \,)$}{
			$(\, B_+^N \, , \, R_+^N \, , \, Belong_+^N \, , \, B_-^N \, , \, R_-^N \, , \, Belong_-^N \,)[1] = (\, A(N) \, , \, 0 \, , \, true \, , \, A(N) \, , \, 0 \, , \, true \,)$
		}
		\Else{
			\For(\tcp*[f]{The first pass}){$i = 1, \dots, t(N)$}{
				$(\, r_{min} \, , \, r_{max} \,) \, \,=\, (\, \max( \,0 \, , \, i-t(N_l) \,) \,\, , \, \min(\, t(N_r) \, , \, i \,) \,)$\;
    
				$(\, B_-^N \, , \, R_-^N \, )[i] \,=\, \minargmin \limits_{r_{min} \,\, \leqslant \,\, %
					r \,\, \leqslant \,\, r_{max}} \, (\, B_-^{N_r}[r] + B_-^{N_l}[i-r] \,)$\;
     
				$(\, B_+^N \, , \, R_+^N \, )[i] \,=\, \maxargmax \limits_{r_{min} \,\, \leqslant \,\, %
					r \,\, \leqslant \,\, r_{max}} (\, B_+^{N_r}[r] + B_+^{N_l}[i-r] \,)$
			}
			\uIf%
			{$(\, N \,\, is \,\, not \,\, the \,\, Root \,)$}{
				\For(\tcp*[f]{The second pass in decreasing order}){$i = t(N), \dots, 1$}{
					\lIf{$(\, A(N)-B_+^N[i-1] \, \geqslant \, B_-^N[i] \,)$}{
						$Belong_-^N[i] \,=\, false$
					}
					\lElse{
						$(\, B_-^N \, , \, R_-^N \, , \, Belong_-^N \,)[i] \,=\, (\, A(N)-B_+^N[i-1] \, , \, R_+^N[i-1] \, , \, true \,)$
					}
					\lIf{$(\, A(N)-B_-^N[i-1] \, \leqslant \, B_+^N[i] \,)$}{
						$Belong_+^N[i] \,=\, false$
					}
					\lElse{
						$(\, B_+^N \, , \, R_+^N \, , \, Belong_+^N \,)[i] \,=\, (\, A(N)-B_-^N[i-1] \, , \, R_-^N[i-1] \, , \, true \,)$
					}
				}
			}
			\Else{
				\lIf{$(\, A(Root) \, \leqslant \, B_+^{Root}[1] \,)$}{
					$Belong_+^{Root}[1] \,=\, false$
				}
				\lElse{
					$(\, B_+^{Root} \, , \, R_+^{Root} \, , \, Belong_+^{Root} \,)[1] \,=\, (\, A(Root) \, , \,0 \, , \, true \,)$
				}
			}
		}
	}
	$\widetilde{\cal H} \, = \, \varnothing$\;
 
	$Q${\LARGE .}$Enqueue(\, Root \, , \, t(Root) \, , \, 0 \,)$\;
 
	\Do{$\, Q \,\, is \,\, not \,\, empty$}{
		$(\, N \, , \, i \, , \, i^{^{_N}}_{_{\cal H}} \,) \, = \, Q${\LARGE .}$Dequeue()$\;
  
		\lIf{$(\, i^{^{_N}}_{_{\cal H}} \,\, is \,\, even \,)$}{
			$(\, belong \, , \, r \,) \, = \, (\, Belong_+^N[i]\, , \, R_+^N[i]\,)$
		}
		\lElse{
			$(\, belong \, , \, r \,) \, = \, (\, Belong_-^N[i]\, , \, R_-^N[i]\,)$
		}
		\lIf(\tcp*[f]{Post-Increment of $i^{^{_N}}_{_{\cal H}}$}){$(\, belong \,)$}{
			$\widetilde{\cal H} \, \leftarrow \, (\, N \, , \, i^{^{_N}}_{_{\cal H}}\!$++$\,)$
		}
		\lIf{$(\, r \, > \, 0 \,)$}{
			$Q${\LARGE .}$Enqueue(\, N_r \, , \, r \, , \, i^{^{_N}}_{_{\cal H}} \,)$
		}
		\lIf(\tcp*[f]{$\!belong\!:\,\,$true\,=\,1\,,\,false\,=\,0}){$(\, r + belong \, < \, i \,)$}{
			$Q${\LARGE .}$Enqueue(\, N_l \, , \, i - belong - r \, , \, i^{^{_N}}_{_{\cal H}} \,)$
		}
	}
	\Return $\widetilde{\cal H}$
	
\end{algorithm2e*}

The d-best subset $\cal H$ is specified from $R^N$ as before. However, since both a node $N$ and nodes that are nested in it, may be included in $\cal H\,,$ we added an indicator $Belong_+^N[i]\,/\,Belong_-^N[i] \quad i \!=\! 1, \dots, t(N)\,$ (boolean vector variables stored in a node $N$), indicating whether $N$ belongs to ${\cal H}_+^N\!(i)\,/\,{\cal H}_-^N\!(i).$ In addition, for every node $N \!\in\! \cal H,$ the index $i^{^{_N}}_{_{\cal H}}$ (Section \ref{sec:SegFrBk}) is calculated, so Algorithm \ref{alg:d1} returns a subset $\,\widetilde{\cal H}\,,$ which is a subset of the pairs $(N\,,\, i^{^{_N}}_{_{\cal H}}).$
\begin{remark}\label{remark:MinimalBestSubset}
	
$\cal H$ is not necessarily of  minimal size, and in extreme case, when the number of nodes with positive attribute is too small,  it does not provide the best benefit.  (See Remark \ref{remark:BestSubsetsWithSameBenefitAndDifferentSizes} and Note \ref{itm:algoc} in section \ref{sec:auxAlgC}). The best subset with the best benefit and minimal size, is always associated with the maximal value in $B^{Root}\!.$
It can be specified by running the queue starting from $Q${\LARGE .}$Enqueue(\, Root \, , \, k' \,)$ ($\,k'$ replaces $t(Root)\,$), where $k'$ is the minimal index such that the value $B^{Root}[k']$ is the maximal in $B^{Root}\!.$
\end{remark}

\subsubsection{Time complexity}\label{sec:complexity}

For our auxiliary algorithms, the vector variable size is  bounded by $k.$ 
The vector variables of a node $N$ may be calculated in $O(\min (k,|{\cal L}^{N_r}|) \!\cdot\! \min (k,|{\cal L}^{N_l}|))$ time. For the common case, where $k<< |{\cal L}|$, this amounts to $O(k^2)$, and is independent of the tree size. 
The algorithm linearly scans all the nodes and requires $O(|{\cal L}| \cdot (\! \min (k, \log |{\cal L}|) )^2\,)$ 
time.
This includes the time required to get the best subset from the node vectors.

The full algorithm starts by calculating the node dimensions $(b^{^{_N}},f^{^{_N}}).$ First, these dimensions are calculated for the leaves of $\,\cal T$ in $O(|I|)$ time, and then propagated to the rest of the nodes in linear time.  Overall, this calculation takes $O(|I|+|{\cal T}|) = O(|I|+|{\cal L}|)$ time.

Thus, the total time complexity is $O(|I| \, + \, n \cdot |{\cal L}| \cdot \big{(} \! \min (k, \log |{\cal L}|) \big{)}^2\,)$, where $n$ is the number of iterations made by scheme \ref{sch:main}.
The straightforward (and least tight) upper-bound on $\,n\,$ is the number of segmentations $s \!\in\! \cal S$ with different scores ${\cal M}(s)$ (the measure maximized in scheme \ref{sch:main}), since ${\cal M}(s)$ strictly increases from iteration to iteration (Section \ref{sec:ourapproach}). However, in practice, we found that only a few iterations are required (no more than five).

\subsubsection{The best segmentation specified by a subset of unlimited size}\label{sec:best_pruning}

Sometimes, we are interested in a segmentation $s \!\in\! \cal S$ achieving the best score ${\cal M}(s),$ regardless of the size of
a subset ${\cal N}_s.$ Then, the auxiliary algorithm becomes linear, and is  significantly simpler. Lemma \ref{lemma:consist:equiv} implies that in this case, optimizing $\,\cal M$ yields $s \!\in\! \cal S$ with the same score ${\cal M}(s)$, for each of the consistency types b/c/d. By simply discarding the node subset size parts, the b-consistency algorithm can be simplified to be  particularly efficient. Algorithm \ref{alg:best_val} provides the full description.

In every node $N,$ we store only the maximal benefit over all b-best subsets in ${\cal T}^N\!,$ regardless of their sizes. That is, we need only a scalar variable $p^N\!,$ storing the maximal value in the vector $B^N$ in Algorithm \ref{alg:b1}. After the values $p^N\!$ are calculated for all nodes, the b-best subset $\cal H$ is found as the optimal cut of $\,\cal T.$ 
In this case, $\cal H$ has the minimal size (see Remark \ref{remark:MinimalBestSubset}), {\em i.e.}, there is no b-best subset in $\cal T,$ that has the same benefit, while being smaller.

Processing of each node is in $\,O(1)\,;$ hence, the time complexity of this algorithm is $O(|{\cal T}|) = O(|{\cal L}|)$. Note that Algorithm \ref{alg:best_val} returns two subsets: $\cal H$ and $\cal G \!\subset\! H$ (the nodes with a positive attribute).

\begin{algorithm2e}[!h]
	\caption{Auxiliary algorithm for finding the best segmentation specified by an unlimited subset}
	\label{alg:best_val}
	\DontPrintSemicolon
	
	\For{$N = N_1, \, N_2, \, N_3, \, \dots, \, N_{|{\cal{T}}|}$}{
		\If{$(\, N \,\, is \,\, a \,\, leaf \,)$}{
			$p^N = \max (\, A(N)\,,\,0 \,)$
		}
		\lElse{
			$p^N = \,p^{N_r} +\, p^{N_l}$
		}
	}
	$(\, \cal H, \, G \,) \, = \, (\, \varnothing, \, \varnothing \,)$\;
	$Q${\LARGE .}$Enqueue(\, Root \,)$\;
	\Do{$\, Q \,\, is \,\, not \,\, empty$}{
		$N \, = \, Q${\LARGE .}$Dequeue()$\;
		\uIf%
		{$(\, \max (\, A(N)\,,\,0 \,) \, \geqslant \, p^N \,)$}{
			\lIf{$(\, A(N) > 0 \,)$}{
				${\cal G} \, \leftarrow \, N$
			}
			${\cal H} \, \leftarrow \, N$
		}
		\Else{
			$Q${\LARGE .}$Enqueue(\, N_r \,)$\;
			$Q${\LARGE .}$Enqueue(\, N_l \,)$
		}
	}
	\Return $(\, \cal H, \, G \,)$
	
\end{algorithm2e}

\subsubsection{Auxiliary algorithms' correctness}
\label{sec:Correctness}

\begin{theorem}
    The auxiliary algorithms optimize the auxiliary measure \eqref{eq:Pj_attr} subject to the corresponding consistency, and the constraint on the maximal number of nodes in ${\cal N}_s$.  
\end{theorem}
As each of the auxiliary algorithms recursively applies Lemma \ref{lemma:conclBest}, the  proof readily follows by induction on $Height(\cal T).$

\subsubsection{A note on the implementation}\label{sec:accuracy}

Each of the auxiliary algorithms calculates the benefit of node subsets by performing arithmetic operations with natural numbers $b^{^{_N}}\!,f^{^{_N}}\!$ and the real number $\omega$. To avoid numerical error in the accumulation, we use integer arithmetic. We represent the benefit with two natural numbers, each of which is a linear combination of $\,b^{^{_N}}\!$ and  $f^{^{_N}}\!$ values, with $\,\pm 1$ coefficients. To compare the benefits of different subsets, we need only a single operation involving $\omega$. 

\section{Experiments}
\label{sec:experiments}

The contribution of this paper is mostly theoretical, in providing, for the first time,  effective algorithms for bounding the obtainable  Jaccard index quality of a segmentation. 
These bounds, depending on the hierarchy, the consistency, and the number of nodes, are experimentally illustrated below.

To the best of our knowledge, the optimization of the Jaccard index was not considered before, which  prevents us from comparing our empirical results with prior work.

\def\myfigsize{0.19}
\begin{figure*}[tbh]
    \centering
	\begin{tabular}[c]{ccc}
		\centering
		\begin{subfigure}[b]{\myfigsize\textwidth}
			\centering
			\includegraphics[width=\textwidth]{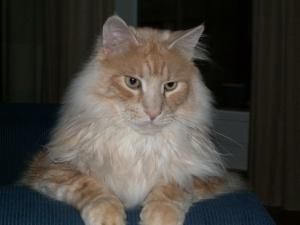}
			\caption{\footnotesize An Image of a cat}
		\end{subfigure}
		&
		\begin{subfigure}[b]{\myfigsize\textwidth}
			\centering
			\includegraphics[width=\textwidth]{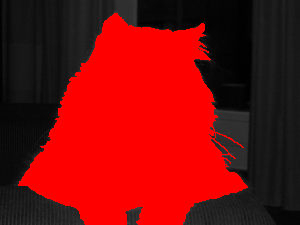}
			\caption{\footnotesize The Ground Truth}
		\end{subfigure}
		&
		\begin{subfigure}[b]{\myfigsize\textwidth}
			\centering
			\includegraphics[width=\textwidth]{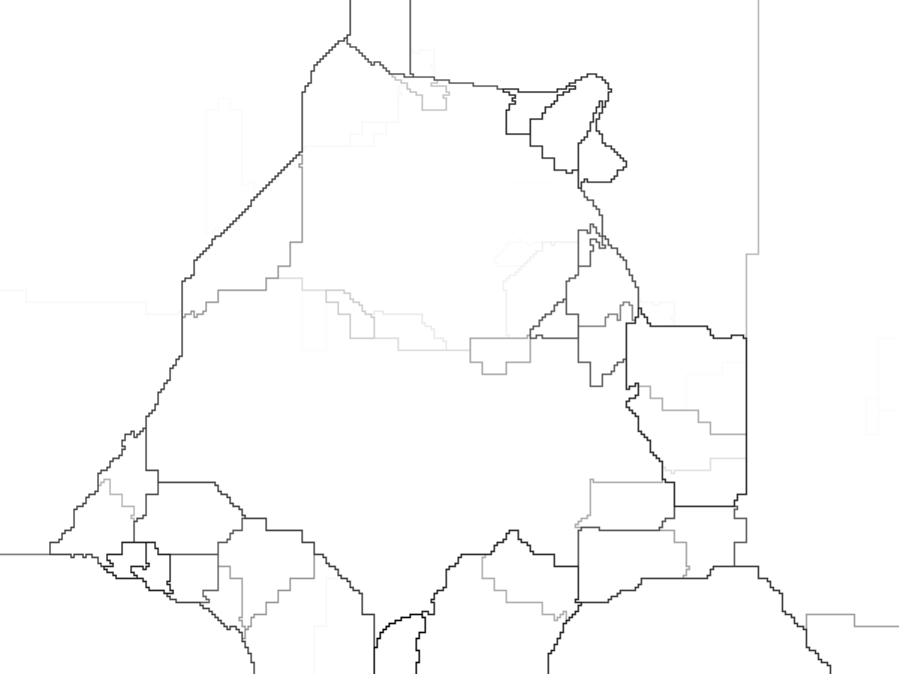}
			\caption{\footnotesize Saliency Map : HED}
		\end{subfigure}
        \\
		\begin{subfigure}[b]{\myfigsize\textwidth}
			\centering
			\includegraphics[width=\textwidth]{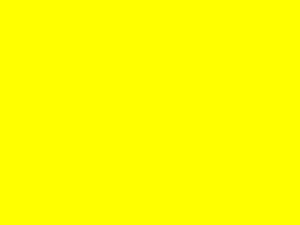}
            \caption{\footnotesize $J_{_b}[{\color{red} 1}] = 0.44$}
		\end{subfigure}
		&
		\begin{subfigure}[b]{\myfigsize\textwidth}
			\centering
			\includegraphics[width=\textwidth]{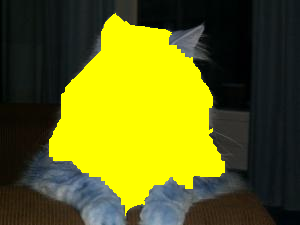}
			\caption{\footnotesize $J_{_c}[{\color{red} 1}] = 0.71$}
		\end{subfigure}
		&
		\begin{subfigure}[b]{\myfigsize\textwidth}
			\centering
			\includegraphics[width=\textwidth]{Images/hed_cat_c_1.png}
			\caption{\footnotesize $J_{_d}[{\color{red} 1}] = 0.71$}
		\end{subfigure}
        \\
		\begin{subfigure}[b]{\myfigsize\textwidth}
			\centering
			\includegraphics[width=\textwidth]{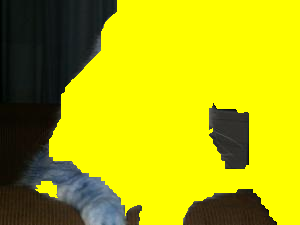}
            \caption{\footnotesize $J_{_b}[{\color{red} 5}] = 0.52$}
		\end{subfigure}
		&
		\begin{subfigure}[b]{\myfigsize\textwidth}
			\centering
			\includegraphics[width=\textwidth]{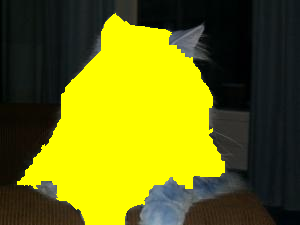}
			\caption{\footnotesize $J_{_c}[{\color{red} 5}] = 0.83$}
		\end{subfigure}
		&
		\begin{subfigure}[b]{\myfigsize\textwidth}
			\centering
			\includegraphics[width=\textwidth]{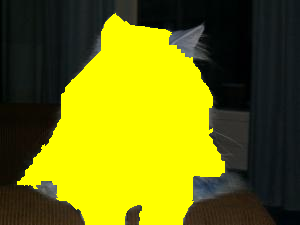}
			\caption{\footnotesize $J_{_d}[{\color{red} 5}] = 0.89$}
		\end{subfigure}
		\\
		\begin{subfigure}[b]{\myfigsize\textwidth}
			\centering
			\includegraphics[width=\textwidth]{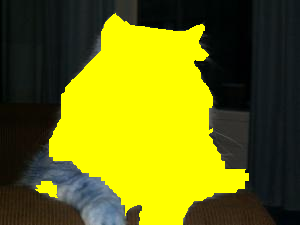}
            \caption{\footnotesize $J_{_b}[{\color{red} 10}] = 0.83$}
		\end{subfigure}
		&
		\begin{subfigure}[b]{\myfigsize\textwidth}
			\centering
			\includegraphics[width=\textwidth]{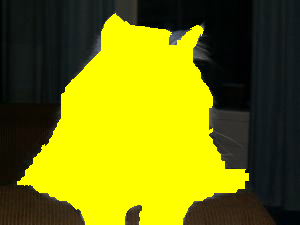}
			\caption{\footnotesize $J_{_c}[{\color{red} 10}] = 0.94$}
		\end{subfigure}
		&
		\begin{subfigure}[b]{\myfigsize\textwidth}
			\centering
			\includegraphics[width=\textwidth]{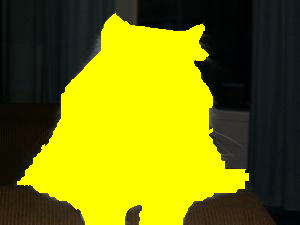}
			\caption{\footnotesize $J_{_d}[{\color{red} 10}] = 0.96$}
		\end{subfigure}
		\\
		\begin{subfigure}[b]{\myfigsize\textwidth}
			\centering
			\includegraphics[width=\textwidth]{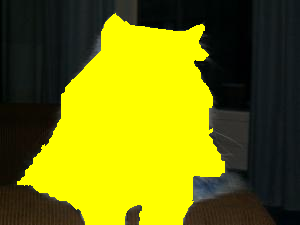}
            \caption{\footnotesize $J_{_b}[{\color{red} 20}] = 0.93$}
		\end{subfigure}
		&
		\begin{subfigure}[b]{\myfigsize\textwidth}
			\centering
			\includegraphics[width=\textwidth]{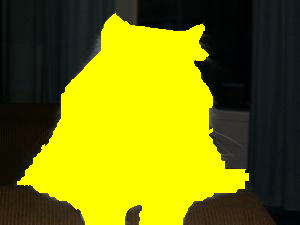}
			\caption{\footnotesize $J_{_c}[{\color{red} 20}] = 0.95$}
		\end{subfigure}
		&
		\begin{subfigure}[b]{\myfigsize\textwidth}
			\centering
			\includegraphics[width=\textwidth]{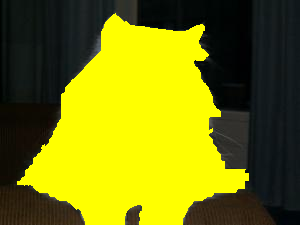}
			\caption{\footnotesize $J_{_d}[{\color{red} 20}] = 0.96$}
		\end{subfigure}
	\end{tabular}
	\caption{Hierarchy consistent optimal segmentations for the HED hierarchy. (a) original image (a cat image from the Weizmann database), (b) ground truth, (c) the saliency map for the HED hierarchy. The segmentations are calculated for the three b-, c-, and d-consistencies and for several numbers of nodes. Note that for a low number of nodes ({\em e.g.}, 5) the b-consistent segmentation (g) is of lower quality than the other segmentations (h)(i). Note also that the c-consistent segmentation (h) is slightly worse than the d-consistent one (i). The differences decrease when the number of nodes increases.\label{fig:examples}}
\end{figure*}

\begin{figure*}[tbh]
	\centering
	\begin{tabular}{ccc}
	      \includegraphics[width=0.3\textwidth]{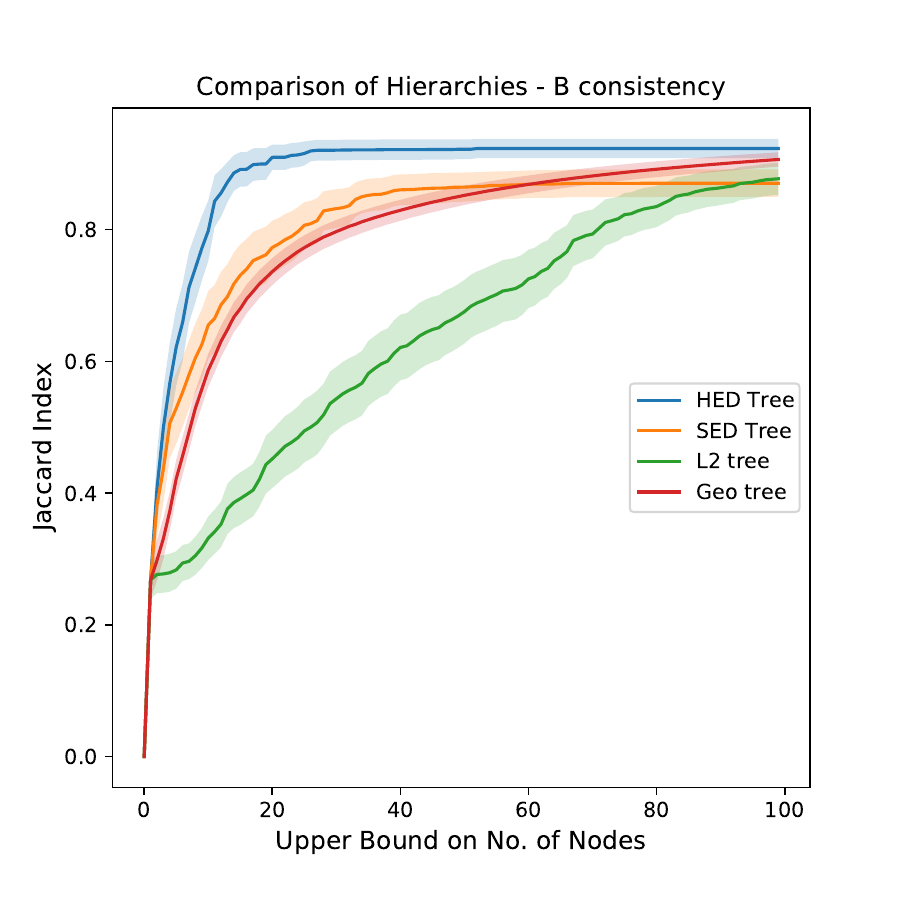}
	     &  
\includegraphics[width=0.3\textwidth]{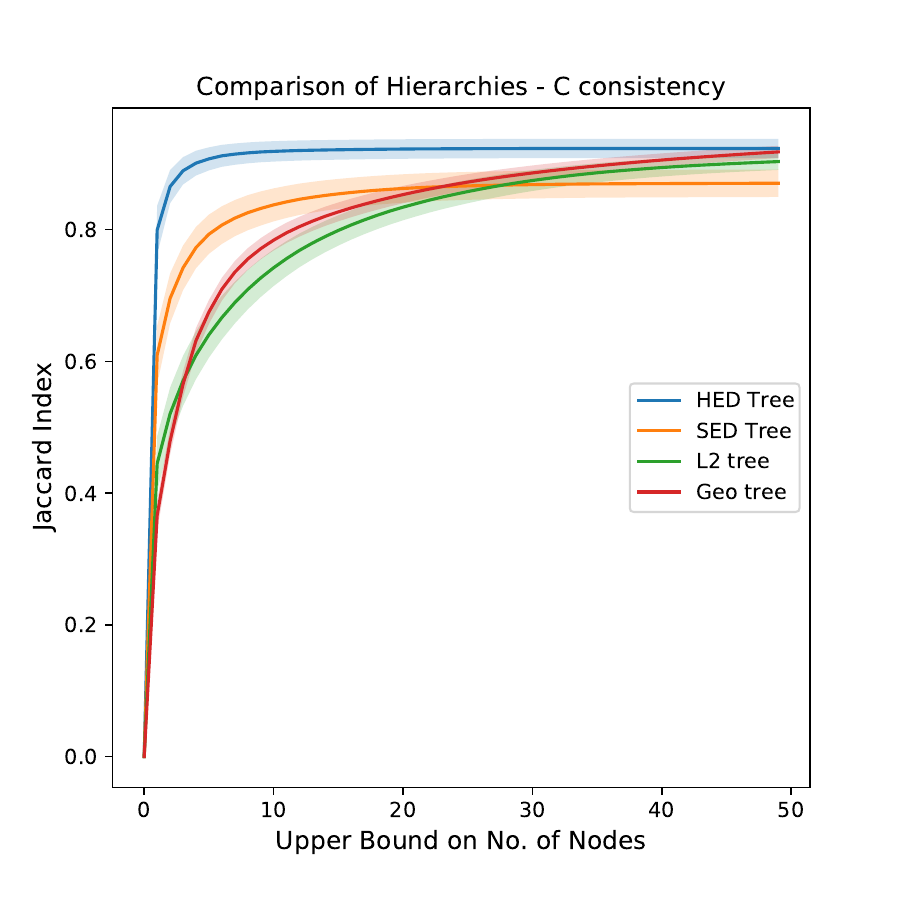}
	     & 
\includegraphics[width=0.3\textwidth]{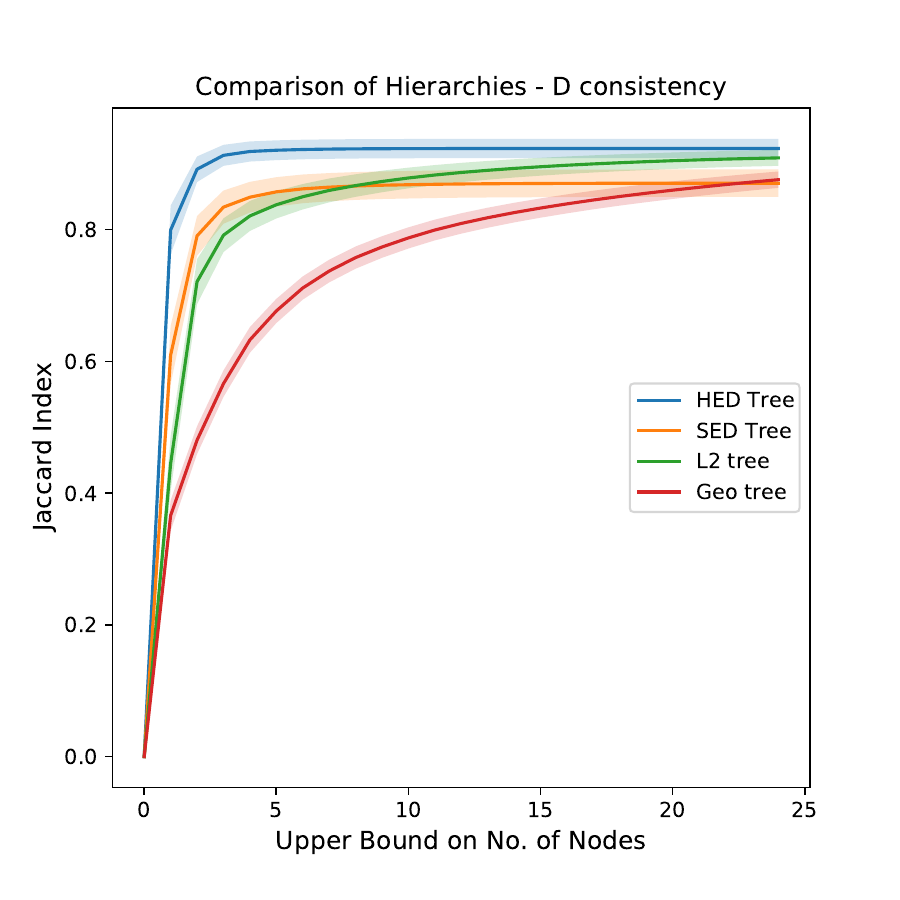}	\end{tabular}
\caption{An illustration of the maximal Jaccard index, obtainable for a given number of nodes. Every one of the plots correspond to one of the segmentation-hierarchy consistencies. The curves correspond to averages over all images in the Weizmann DB, to the four hierarchies.  Filtered hierarchies are used. The use of d-consistency clearly requires a significantly lower number of nodes for the same quality, relatively to the c-consistency, which, in turn, requires a lower number of nodes than the usage of the b-consistency. Also, as expected, the hierarchy built using the HED edge detector gives better results than the other hierarchies demonstrated here. \label{fig:experiments-filtered}}
\end{figure*}
	
For the experiments, we consider four BPT hierarchies. The first, denoted geometric tree, is image independent and  serves as a baseline. The other three hierarchies are created as the minimum spanning tree of a super-pixels graph, where the nodes are SLIC superpixels  \cite{achanta2012slic}. The weights of the graph are specified by different types of gradients.  
More specifically, we consider the following hierarchies:  
\begin{enumerate}
\item Geometric Tree (image-independent baseline) -- Starting from the root node (the entire image), each node split into two equal parts (the node children), horizontally or vertically, depending on whether the height or the width of the node is larger. 
Note that the geometric tree is independent of the image content.

\item L2 Tree –- based on traditional, low quality non-learned gradient: the L2 difference between the RGB color vectors. 

\item SED Tree -- based on learned,  Structured Forests Edge detection, which can be considered medium quality \cite{6751339}. 

\item HED Tree -- Modern, high quality, deep learning based, Holistically-Nested Edge Detector~\cite{xie15hed}. 
\end{enumerate}

A common issue with hierarchical image segmentation is the presence of small regions (containing few pixels) at lower depths in the hierarchy. These small regions are found more frequently when generating the HED  and SED trees, as their gradient generally contains thick boundaries.
It is therefore common to filter the hierarchy and to remove such small unwanted regions; see, {\em e.g.}, the implementation of \cite{felzenszwalb2004efficient} and \cite{baltaxe2016local, perret2019removing}. We followed this practice and use the Higra \cite{perret2019higra} area-based filtering algorithm proposed in \cite{perret2019removing}.

The leaves of the image independent,  geometric-tree are the image pixels, which makes this tree large (and regular). The other trees are smaller, as they use super pixels, and also benefit from the filtering process, when applied. 

We calculated the best segmentations that match the different hierarchies, and show how they depend on the particular hierarchy that is used, and on the consistency type. First, we show several examples of such best segmentations, corresponding to the same image, using the HED hierarchy; see Figure \ref{fig:examples}. 
As expected, the segmentation quality improves with the number of nodes that are used, and with the consistency type ($b<c<d$). 

Figure \ref{fig:experiments-filtered} confirms this observation, and shows that the average Jaccard index, over an image dataset, grows with the number of hierarchy nodes. It also shows that requiring  d-consistency allows us to use a relatively small number of nodes for getting good segmentation, with a high Jaccard index. C-consistency follows, and b-consistency is last.  The differences between the consistencies are clearly seen in Figure
\ref{fig:experiments-filtered-byhierarchy}.
It is also clear that better hierarchies, obtained with more accurate edge detectors, provide much higher quality of segmentation with lower number of nodes.
These plots show the average Jaccard index over  $100$ images of the Weizmann database \cite{alpert2012image}. Every image in this database contains a single object over a background, which match the applicability of the Jaccard index.

\begin{figure*}[tbh]
	\centering
	\begin{tabular}{cc}
	      \includegraphics[width=0.3\textwidth]{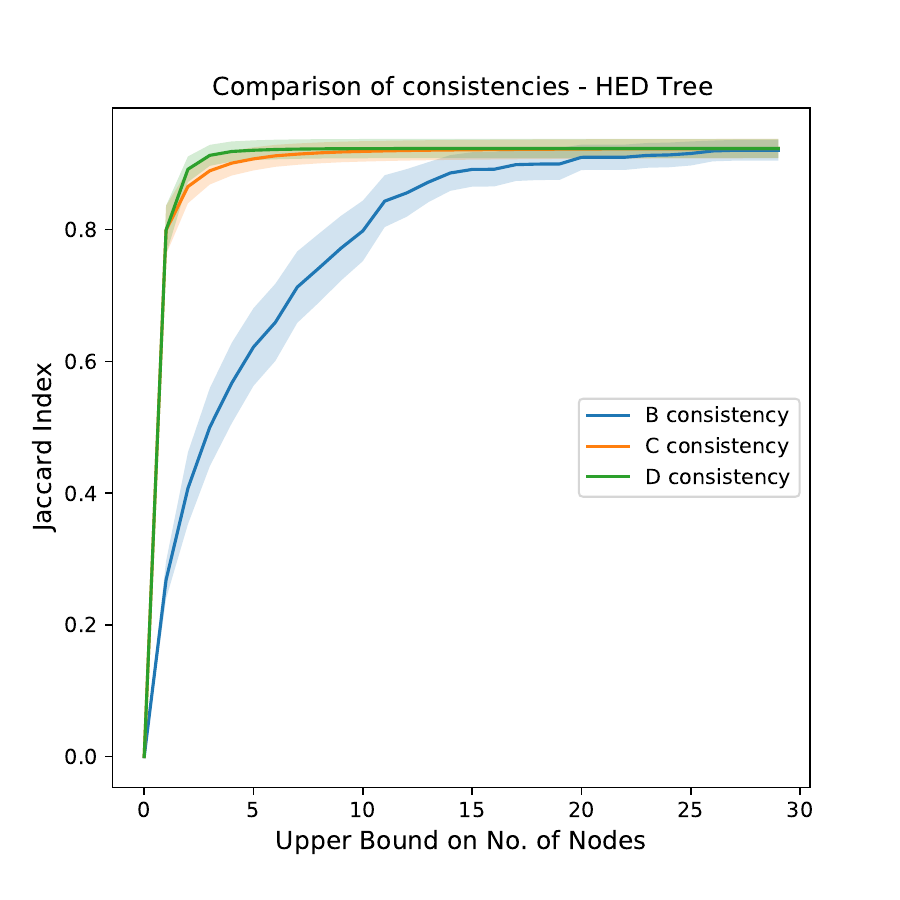}
	     &  
\includegraphics[width=0.3\textwidth]{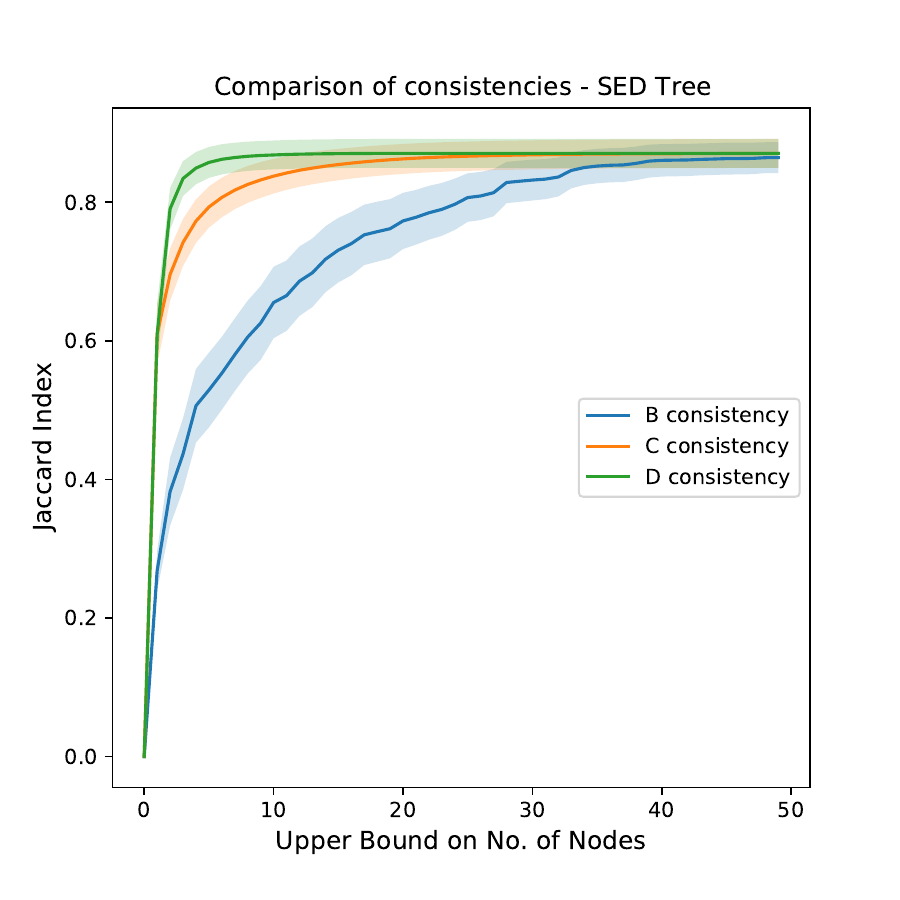} \\
	     
\includegraphics[width=0.3\textwidth]{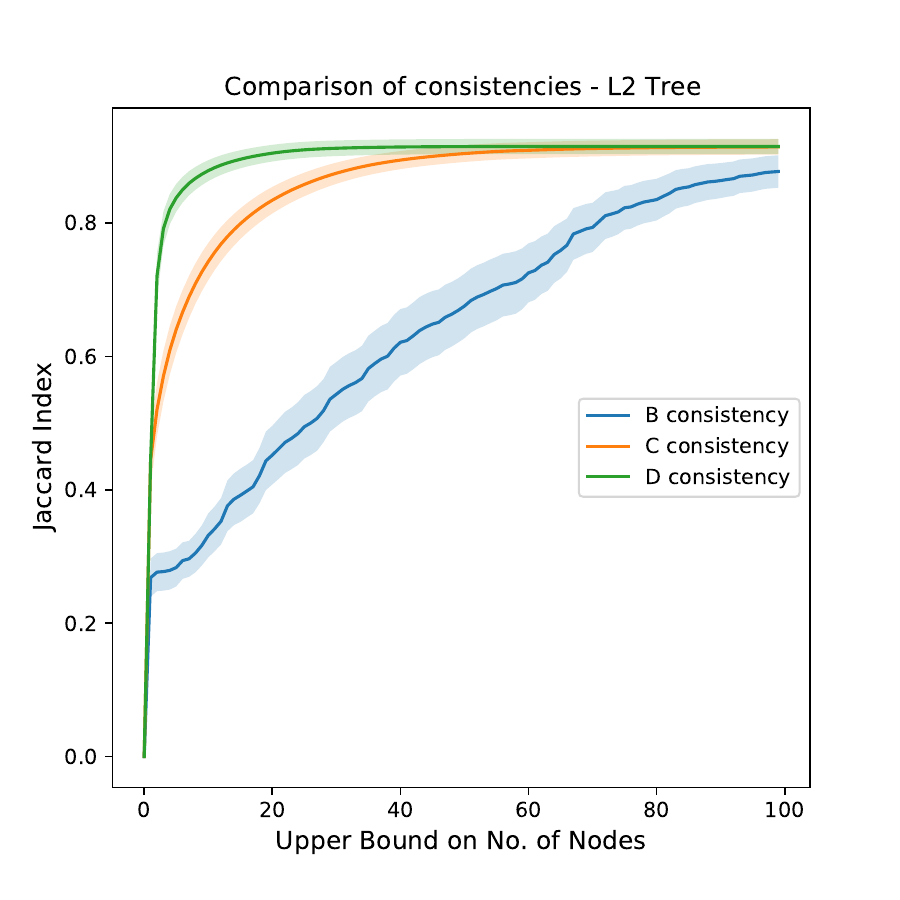} 
        & 
\includegraphics[width=0.3\textwidth]{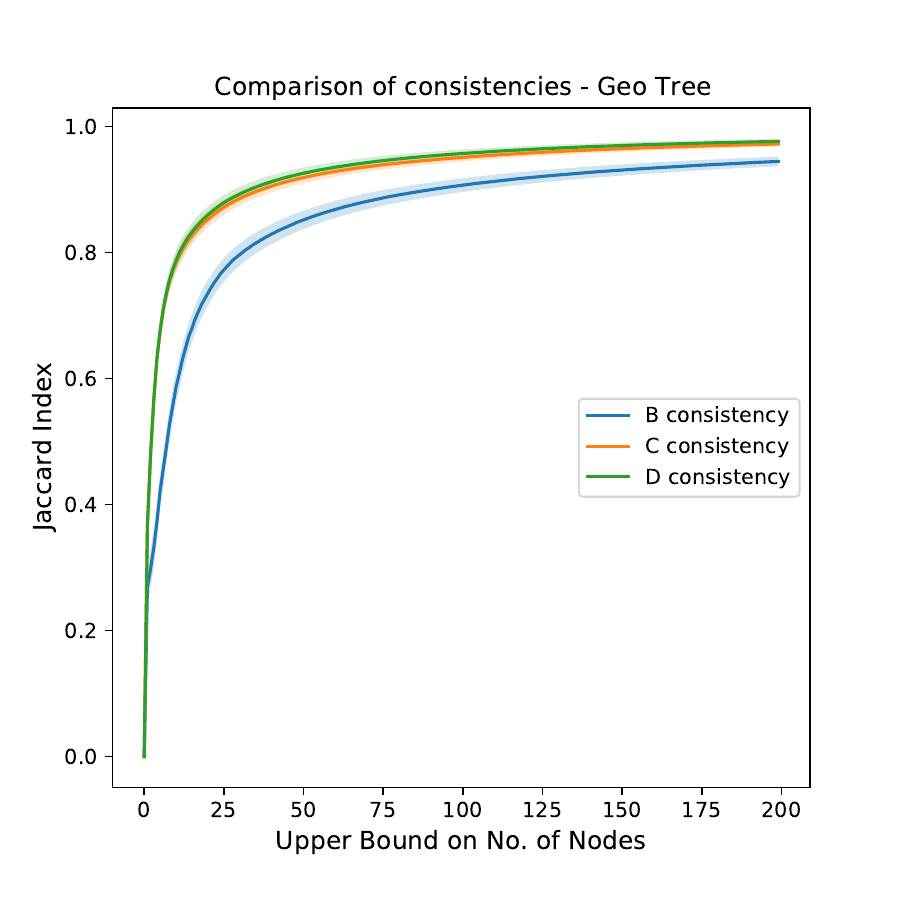}.  \end{tabular}
\caption{Comparing the segmentation quality obtainable using the three segmentation-hierarchy consistencies. The comparison is carried out for each of the four different types of hierarchies. Note that the performance with c and d consistencies is similar for the best, HED, tree. For better visibility, we used a different scale in the x-axis, for each of the hierarchy. \label{fig:experiments-filtered-byhierarchy}}
\end{figure*}

The average Jaccard index curves are smooth. We observed however that for particular images, the curves have stair-like behavior, implying that the same Jaccard index is achieved for different $k$ values, which happens, {\em e.g.}, when adding a few nodes to $N_s$ does not change the foreground specification.

Note that for b-consistency the geometric, image-independent  tree is better than say the L2 tree. This happens because in the L2 tree,  we have many spurious small nodes that are close to the root, even after the filtering process. B-consistency chooses a set of nodes which is a cut in the tree; when taking a cut that contains the important nodes needed to approximate the GT segment, some spurious nodes must be included, which significantly increases the node count ($k$). 

The best results (in terms of lowest node count) are achieved with d-consistency. This holds for all hierarchies. For the higher quality hierarchies, the node count needed for excellent quality is remarkably low (only four on average). 
We found that even if the hierarchy contains errors such as incorrect merges and small nodes near the root, segmentations specified by d-consistency still require a small node count. To illustrate this robustness property, consider the case where one incorrect merge was done; see Figure \ref{fig:faulty-hierarchy}.
This merge leads to a sequence of modes that are not purely foreground or background. In this example, the b-consistent foreground segment is specified by the cut containing 6 nodes (A,B,\dots,F), c-consistency requires one node less (A,B, \dots, E), while d-consistency requires only 2 nodes (K and F). This robustness is significant because the hierarchy is constructed usually by an error-prone, greedy process. By using d-consistency, the harm made by the greedy process can be compensated to some extent. 
Note that when using the geometric tree, the segmentation qualities obtained by c and d consistencies are not very different; see Figure \ref{fig:experiments-filtered-byhierarchy}. The merging errors made by the geometric tree are numerous, and happen in all hierarchy levels, therefore they cannot be corrected by a few set-difference operations.     

The experiments are meant only to be illustrative, and are not the main contributions of this paper. Several surprising findings are observed, however. First, it turned out that for approximating a segment, in the Jaccard index sense, the geometric tree provide reasonable results, which are often as good as some of the others trees (but not of the modern HED tree). Note that while all the nodes in this case are image-independent rectangles, the nodes that were selected for the approximation are based on the (image-dependent) ground truth segmentation. 
We also found that the hierarchies based on the SED edge detector are not as good as we could have expected. This was somewhat surprising because previous evaluations of the SED show good results (F-number=0.75, on BSDS \cite{6751339}). Overall, these results imply that hierarchies built greedily are sensitive to the gradient that is used. 

\begin{figure*}[tbh]
	\centering
	\begin{tikzpicture}[node distance={15mm}, main/.style = {draw, circle}] 
\node[main,fill=green] (E) {E}; 
\node[main,fill=red] (F) [right of=E]{F}; 
\node[main,fill=yellow] (G) [above left of=F] {G}; 
\node[main,fill=yellow] (H) [above left of=G] {H}; 
\node[main,fill=green] (D) [left of=G] {D}; 
\node[main,fill=yellow] (H) [above left of=G] {H}; 
\node[main,fill=yellow] (I) [above left of=H] {I}; 
\node[main,fill=yellow] (J) [above left of=I] {J}; 
\node[main,fill=yellow] (K) [above left of=J] {K}; 
\node[main,fill=green] (C) [left of=H] {C}; 
\node[main,fill=green] (B) [left of=I] {B}; 
\node[main,fill=green] (A) [left of=J] {A};  

\draw (E) -- (G);
\draw (F) -- (G);
\draw (G) -- (H);
\draw (D) -- (H);
\draw (H) -- (I);
\draw (I) -- (J);
\draw (C) -- (I);
\draw (J) -- (K);
\draw (B) -- (J);
\draw (A) -- (K);

\end{tikzpicture} 

\caption{Consistency-robustness against incorrect mergings -- An example of a hierarchy, with several nodes in the foreground (A,\dots,E) and one node in the background (F), which is merged incorrectly with E. Expressing the foreground using this hierarchy requires 6,5, and 2 nodes in the b-,c-, and d-consistency, respectively; see text. \label{fig:faulty-hierarchy}}
\end{figure*}
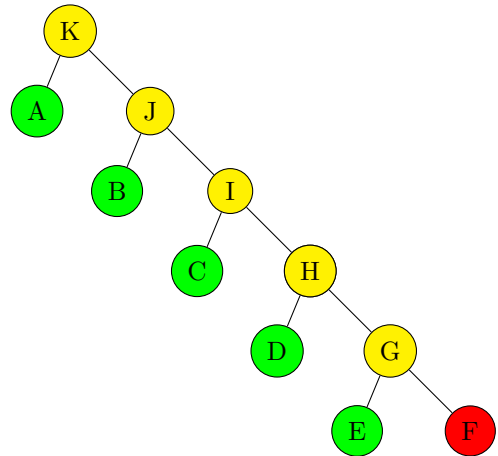

\section{Conclusions}\label{sec:conclusions}

This paper considered the relation between the hierarchical representation of an image and the segmentation of this image. It proposed that a segmentation may depend on the hierarchy in 4 different ways, denoted consistencies. The higher level consistencies are more robust to hierarchy errors, which allows us to describe the segmentation in a more economical way, use fewer nodes, relative to the lower-level consistencies that are commonly used. 

While the common a-consistency requires that every segment is a separate node in a hierarchy cut, using b-consistency allows to describe segments that were split between different branches of the hierarchy. The c- and d-consistency no longer require that the segmentation is specified by a cut, and this way can ignore, non-important small nodes. The d-consistency can even compensate for incorrect merges that occurred in the (usually greedy) construction of the hierarchy. We found, for example, that fairly complicated segments 
can be represented by only 3-5 nodes of the tree, using the hierarchy built with a modern edge detector (HED \cite{xie15hed}) and d-consistency. This efficient segment representation opens the way to new algorithm  for analyzing segmentation and searching for the best one. Developing such algorithms seems nontrivial and is left for future work.

The number of nodes required to describe a segmentation is a measure of the quality of the hierarchy. A segmentation may be accurately described by a large number of leaves of almost any hierarchy. For describing the segmentation with a few nodes, however, the hierarchy should contain nodes that correspond to the true segments, or at least to a large fraction of them. Thus, this approach is an addition to the variety of existing tools that were proposed for hierarchy evaluation. 

Technically, most of this paper was dedicated to deriving rigorous and efficient algorithms for optimizing the Jaccard index. For this complex optimization, the co-optimality tool was introduced. We argue that with this tool, other measures of segmentation quality, such as the boundary-based $F_b$ measure \cite{martin2003empirical} considered in  \cite{pont2012supervised}, may be optimized more efficiently, and propose that for future work as well.  

\bmhead{Acknowledgments}
This work was done when the first author was with the Math dept., Technion, Israel and the second author was with CSE. dept., IIT Delhi.

\bibliography{report2}

\begin{appendices}

\section{Proof of Lemma 1}\label{appendix:B}

\noindent 
{\bf Lemma \ref{lemma:coarsest-T-partition}}  (repeated) 

\begin{enumerate}
\item[\it i.] A $\,\cal T$-partition of a pixel subset $\,Y \! \subset I\,$ is non-coarsest, if and only if, it contains a non-coarsest $\,\cal T$-partition of some node $\,N \! \in \! \cal T$ that is included in $\,Y$ ($\,N \! \subset Y$).
			
\item[\it ii.] When the coarsest $\,\cal T$-partition of a pixel subset $\,Y \! \subset I\,$ exists, it is unique.
\end{enumerate}

\begin{proof}[{\bf Proof of Lemma \ref{lemma:coarsest-T-partition}}\,{\bf :}]\label{lemma:coarsest-T-partition:Proof}
~\\~
	\begin{enumerate}
		
		\item[\it i.] Let $\,\cal N \! \subset \! T\,$ be a $\,\cal T$-partition of $\,Y \! \subset I\,$.
		
		\item[$\pmb \implies$] Suppose that $\,\cal N$ is non-coarsest. Let $\,\cal N' \! \subset \! T$ be the coarsest $\,\cal T$-partition of $\,Y \! \subset I\,$ ($\,|{\cal N}'| \, < |{\cal N}|\,$).
Any two nodes are either nested or disjoint, hence, $\cal N$ is finer than $\cal N'$: $\,\cal N \leqslant N'\,$, {\em i.e.}, every node of $\,\cal N$ is included in some node of $\,\cal N'$ (otherwise the size of $\,\cal N'$ can be reduced which contradicts that $\,\cal N'$ is the coarsest). Hence, there exists a node $N$ in $\cal N'$ that contains several nodes of $\cal N$, {\em i.e.}, $\cal N$ contains a non-coarsest $\cal T$-partition of $N$.

\item[$\pmb \impliedby$] Suppose that $\cal N$ contains a non-coarsest $\cal T$-partition of some node $N \! \subset Y$. Replacing this $\cal T$-partition of $N$ by $N$ itself yields another $\,\cal T$-partition of $\,Y$, which is coarser then $\,\cal N$. Hence, $\,\cal N$ is non-coarsest.
		
\item[\it ii.] If $\,\cal N$ and $\cal N'$ are two coarsest $\,\cal T$-partitions of $\,Y \! \subset I$, then the size of each of them is  minimal, which implies that $\,\cal N \geqslant N'\,$ and $\,\cal N \leqslant N'$ (since any two nodes either nested or disjoint). Hence, $\,\cal N = N'$.
\end{enumerate}
\end{proof}

\end{appendices}

\end{document}